\RequirePackage[l2tabu,orthodox]{nag}
\documentclass
[letterpaper,12pt,]
{article}

\usepackage{custom}

\usepackage{etex}
\usepackage{verbatim}
\usepackage{xspace,enumerate}
\usepackage[dvipsnames]{xcolor}
\usepackage[T1]{fontenc}
\usepackage[full]{textcomp}
\usepackage[american]{babel}
\usepackage{mathtools}
\usepackage{amsthm}
\usepackage[
letterpaper,
top=1in,
bottom=1in,
left=1in,
right=1in]{geometry}
\usepackage{newpxtext} 
\usepackage{textcomp} 
\usepackage[varg,bigdelims]{newpxmath}
\usepackage[scr=rsfso]{mathalfa}
\usepackage{bm} 
\let\mathbb\varmathbb
\usepackage{microtype}
\usepackage[pagebackref,colorlinks=true,urlcolor=blue,linkcolor=blue,citecolor=OliveGreen]{hyperref}
\usepackage[capitalise,nameinlink]{cleveref}
\crefname{lemma}{Lemma}{Lemmas}
\crefname{fact}{Fact}{Facts}
\crefname{theorem}{Theorem}{Theorems}
\crefname{corollary}{Corollary}{Corollaries}
\crefname{claim}{Claim}{Claims}
\crefname{example}{Example}{Examples}
\crefname{algorithm}{Algorithm}{Algorithms}
\crefname{problem}{Problem}{Problems}
\crefname{definition}{Definition}{Definitions}
\crefname{exercise}{Exercise}{Exercises}
\usepackage{amsthm}

\newtheorem{theorem}{Theorem}[section]
\newtheorem*{theorem*}{Theorem}
\newtheorem{lemma}[theorem]{Lemma}
\newtheorem*{lemma*}{Lemma}

\newtheorem*{fact*}{Fact}
\newtheorem{proposition}[theorem]{Proposition}
\newtheorem*{proposition*}{Proposition}
\newtheorem{corollary}[theorem]{Corollary}
\newtheorem*{corollary*}{Corollary}

\newtheorem*{hypothesis*}{Hypothesis}
\newtheorem{conjecture}[theorem]{Conjecture}
\newtheorem*{conjecture*}{Conjecture}
\theoremstyle{definition}
\newtheorem{definition}[theorem]{Definition}
\newtheorem*{definition*}{Definition}

\newtheorem*{construction*}{Construction}
\newtheorem{example}[theorem]{Example}
\newtheorem*{example*}{Example}

\newtheorem*{question*}{Question}
\newtheorem{algorithm}[theorem]{Algorithm}
\newtheorem*{algorithm*}{Algorithm}

\newtheorem*{assumption*}{Assumption}
\newtheorem{problem}[theorem]{Problem}
\newtheorem*{problem*}{Problem}

\newtheorem*{openquestion*}{Open Question}

\newtheorem*{model*}{Model}
\theoremstyle{remark}

\newtheorem*{claim*}{Claim}
\newtheorem{remark}[theorem]{Remark}
\newtheorem*{remark*}{Remark}

\newtheorem*{observation*}{Observation}
\usepackage{paralist}
\let\originalleft\left
\let\originalright\right
\renewcommand{\left}{\mathopen{}\mathclose\bgroup\originalleft}
\renewcommand{\right}{\aftergroup\egroup\originalright}
\usepackage{turnstile}
\usepackage{mdframed}
\usepackage{tikz}
\usepackage{caption}
\DeclareCaptionType{Algorithm}
\usepackage{newfloat}
\usepackage{xparse}
\usepackage{amsthm} 
\makeatletter
\let\latexparagraph\paragraph
\RenewDocumentCommand{\paragraph}{som}{%
  \IfBooleanTF{#1}
  {\latexparagraph*{#3}}
  {\IfNoValueTF{#2}
    {\latexparagraph{\maybe@addperiod{#3}}}
    {\latexparagraph[#2]{\maybe@addperiod{#3}}}%
  }%
}
\newcommand{\maybe@addperiod}[1]{%
  #1\@addpunct{.}%
}
\makeatother

\usepackage{boxedminipage}

\newcommand{\paren}[1]{(#1)}
\newcommand{\Paren}[1]{\left(#1\right)}

\newcommand{\abs}[1]{\lvert#1\rvert}
\newcommand{\Abs}[1]{\left\lvert#1\right\rvert}

\newcommand{\card}[1]{\lvert#1\rvert}

\newcommand{\Set}[1]{\left\{#1\right\}}

\newcommand{\norm}[1]{\lVert#1\rVert}
\newcommand{\Norm}[1]{\left\lVert#1\right\rVert}

\newcommand{\normt}[1]{\norm{#1}_2}
\newcommand{\Normt}[1]{\Norm{#1}_2}

\newcommand{\Snorm}[1]{\Norm{#1}^2}

\newcommand{\iprod}[1]{\langle#1\rangle}

\newcommand{\Esymb}{\mathbb{E}}
\newcommand{\Psymb}{\mathbb{P}}
\newcommand{\Vsymb}{\mathbb{V}}
\DeclareMathOperator*{\E}{\Esymb}
\DeclareMathOperator*{\Var}{\Vsymb}
\DeclareMathOperator*{\ProbOp}{\Psymb}
\renewcommand{\Pr}{\ProbOp}

\newcommand{\sm}{\setminus}

\newcommand\bdot\bullet

\DeclareMathOperator{\poly}{poly}

\DeclareMathOperator{\polylog}{polylog}
\DeclareMathOperator{\supp}{supp}

\newcommand{\iid}{i.i.d.\xspace}

\newcommand{\Hoelder}{H\"{o}lder\xspace}

\newcommand{\Z}{\mathbb Z}
\newcommand{\N}{\mathbb N}
\newcommand{\R}{\mathbb R}

\newcommand{\cA}{\mathcal A}

\newcommand{\bbP}{\mathbb P}

\renewcommand{\leq}{\leqslant}

\renewcommand{\geq}{\geqslant}
\renewcommand{\ge}{\geqslant}
\let\epsilon=\varepsilon
\numberwithin{equation}{section}

\allowdisplaybreaks
\newcommand*{\Id}{\mathrm{Id}}

\newcommand*{\transpose}[1]{{#1}{}^{\mkern-1.5mu\mathsf{T}}}

\newcommand{\kldiv}[2]{D_{\mathrm{KL}}\Paren{#1 \| #2}}

\title{On the well-spread property and its relation to linear regression\thanks{This project has received funding from the European Research Council (ERC) under the European Union's Horizon 2020 research and innovation programme (grant agreement No 815464).}}

\author{
  Hongjie Chen\thanks{ETH Z\"urich.}
  \and
  Tommaso d'Orsi\footnotemark[2]
}

\begin{document}


\maketitle
\begin{abstract}
  We consider the robust linear regression model $\bm y = X \beta^* +\bm \eta$, where an adversary oblivious to the design $X\in \R^{n\times d}$ may choose $\bm \eta$ to corrupt all but a (possibly vanishing) fraction of the observations $\bm y$ in an arbitrary way.
Recent work \cite{d2021consistent, d2021consistentICML} has introduced efficient algorithms for consistent recovery of the parameter vector. These algorithms crucially rely on the design matrix being well-spread (a matrix is well-spread if its column span is far from any sparse vector).
  
In this paper, we show that there exists a family of design matrices lacking well-spreadness such that consistent recovery of the parameter vector in the above robust linear regression model is information-theoretically impossible.

We further investigate the average-case time complexity of certifying well-spreadness of random matrices. We show that it is possible to efficiently certify whether a given $n$-by-$d$ Gaussian matrix is well-spread if the number of observations is quadratic in the ambient dimension. We complement this result by showing rigorous evidence ---in the form of a lower bound against low-degree polynomials--- of the computational hardness of this same certification problem when the number of observations is $o(d^2)$.
\end{abstract}

\clearpage

\microtypesetup{protrusion=false}
\tableofcontents{}
\microtypesetup{protrusion=true}

\clearpage

\section{Introduction}\label{section:introduction}
For a subspace $V\subseteq \R^{n}$, the well-spreadness property describes how close sparse vectors are to it. 

\begin{definition}[Well-spreadness] \label[definition]{definition:spreadness}
  A subspace $V\subseteq \R^{n}$ is $m$-spread if for any $v\in V$ and any $S\subseteq [n]$ of size $\card{S}\geq n-m$, we have
  \begin{align*}
    \normt{v_{S}} \geq \LB{1} \cdot \normt{v}\,,
  \end{align*}
  where $v_s$ denotes the projection of $v$ onto the coordinates in $S$. We say that a matrix is $m$-spread if its column span is.
\end{definition}
 
Due to its connection to distortion \cite{guruswami2010almost}, Euclidean section properties \cite{baraniuk2008simple} and restricted isometry properties (RIP) \cite{allen2016restricted, spread-sparse-matrices}, well-spread subspaces have been studied in the context of error-correction over the reals \cite{candes2005decoding, guruswami2008euclidean}, compressed sensing matrices for low compression factors \cite{kashin2007remark, donoho2006compressed} convex geometry \cite{gluskin1984norms, kashin2007remark} and metric embeddings \cite{indyk2007uncertainty}. 
Recently, an unforeseen connection between well-spreadness and oblivious adversarial regression models  has emerged \cite{d2021consistent, d2021consistentICML}.

While relations between  properties of the design matrix and algorithmic guarantees are not new ---restricted eigenvalue condition, restricted isometry property (RIP) and distortion are all known to be sufficient to design efficient algorithms for recovering the encoded sparse vector (see \cite{kashin2007remark, zhang2014lower})--- the connection between well-spreadness and oblivious regression appears intriguing as: (i) there is currently no significant evidence of the necessity of this property for recovery, and (ii) there is no indication of a gap between exponential time and polynomial time algorithms depending on the well-spreadness of the design. Investigating this relation is the main focus of this paper.

\paragraph{Oblivious regression}
Oblivious adversarial models offer a convenient framework to find the weakest assumptions under which one can efficiently recover  structured signal from noisy data with \textit{vanishing error}.\footnote{\textit{Adaptive} adversarial models, where the adversary has access to the data, are not suitable to study these questions  as part of the signal may be removed and hence impossible to reconstruct.} Once the observations are sampled, an adversary is allowed to add arbitrary noise \textit{without} accessing the data and with the additional constraint that   for an $\alpha$ fraction of the observations (possibly vanishing small) the noise must have small magnitude.
In the context of regression this idea can be formalized into the following problem.

\begin{problem}[Oblivious linear regression] \label[problem]{problem:oblivious-regression}
  Given observations\footnote{We use bold face to denote random variables.} $(X_1, \bm{y}_1), \ldots, (X_n, \bm{y}_n)$ following the linear model $\bm{y}_i= \iprod{X_i, \beta^*} + \bm{\eta}_i$, where $X_i\in \R^d$, $\beta^*\in \R^d$, and $\bm \eta_i$ is a symmetrically distributed random variable with $\min_{i\in[n]} \bbP\Set{\Abs{\bm \eta_i}\leq 1} = \alpha$, the goal is to to find an estimator $\hat{\beta}$ for $\beta^*$ achieving small \textit{squared parameter error} $\normt{\hat{\beta}-\beta^*}^2$.\footnote{We remark that we may analogously ask for small squared prediction error $\tfrac{1}{n}\normt{X(\hat{\beta}-\beta^*)}^2$, at the coarseness of this discussion the two may be considered equivalent. We also remark that for small values of $\alpha$, symmetry of the noise is necessary \cite{d2021consistentICML}.}
\end{problem}

We may conveniently think of the (possibly vanishingly small) $\alpha$ fraction of entries of $\bm y$ with small noise as the uncorrupted observations. Moreover, as moments are not required to exist, this noise model captures heavy-tailed distributions.

A flurry of works \cite{tsakonas2014convergence, bhatia2017consistent,  suggala2019adaptive, pesme2020online, d2021consistent, d2021consistentICML} has led to the design of efficient and consistent\footnote{An estimator is said to be consistent if its error tends to zero as the number of observations grows.} algorithms that achieve provably optimal error guarantees and sample complexity for oblivious regression.
The guarantees of these algorithms are rather surprising.
For classical regression with Gaussian noise $\Gau{0}{\sigma^2}$, it is known that the optimal error convergence is $O(\sigma^2\cdot d/n)$ \cite{wainwright2019high}. 
For oblivious regression, efficient algorithms obtain squared parameter error bounded by $O(d/(\alpha^2\cdot n))$ and thus are consistent for $n\geq\higherorder{d/\alpha^2}$. \footnote{More generally, we may assume in \cref{problem:oblivious-regression} that $\min_{i\in[n]} \bbP\Set{\Abs{\bm \eta_i}\leq\tau}=\alpha$ by introducing another parameter~$\tau~>~0$. In this case, the error bound becomes $O(\tau^2d/(\alpha^2\cdot n))$ while the analysis of the error bound is essentially the same. In this paper, we set~$\tau=1$ for simplicity.}
As Gaussian distributions $\Gau{0}{\sigma^2}$ can be modeled as noise in \cref{problem:oblivious-regression} with $\alpha=O(1/\sigma)$, these error convergence rates are the same up to constant factors. In other words, even though \cref{problem:oblivious-regression} allows for a large variety of complicated noise distributions, it is possible to achieve  error guarantees similar to those one would be able to achieve under the special case of Gaussian noise.

It turns out that the catch is in the design matrix $X\in \R^{n\times d}$. Algorithms for oblivious regression require the column span $\text{cspan}(X)$ of $X$ to be well-spread. If  $\text{cspan}(X)$ is $\LB{d/\alpha^2}$-spread, then the above guarantees can be achieved efficiently. On the other hand, \textit{no algorithm} is known to obtain non-trivial error guarantees as soon as the design matrix is only $o(d/\alpha^2)$-spread, even in exponential time. This picture raises an important question concerning the relation between oblivious regression and well-spreadness:

\begin{itemize}
  \item[]
  \emph{
    Is the well-spreadness requirement a fundamental limitation of current algorithms or is there a sharp phase transition in the  landscape of the problem? Is this phase transition a computational or statistical phenomenon? 
  }
\end{itemize}
In this paper we provide, to a large extent, answers to these and related questions.

\subsection{Results}\label{section:results}

\paragraph{Information-theoretic lower bounds regarding well-spreadness}
Our first result is the non-existence of algorithms with non-trivial error guarantees for oblivious regression with design matrices lacking well-spreadness.

\begin{theorem} \label{thm:main-it-concise}
  Let $\alpha = \alpha(n) \in (0,1)$.
  For arbitrary $\gamma=\gamma(n)>0$, there exist:
  \begin{enumerate}
    \item a matrix $X\in\R^{n\times d}$ with $\max \Set{\Omega\Paren{\frac{\log d}{\alpha^2}}\,, \Omega\Paren{\frac{d}{\alpha}}}$-spreadness and $X^\top X=n\cdot\Id$,
    \item a distribution $\mD_\beta$ over $d$-dimensional vectors, and
    \item a distribution $\mD_\eta$ ---independent of $\mD_\beta$--- over $n$-dimensional vectors with independent, symmetrically distributed entries satisfying $\min_{i\in[n]}\bbP_{\bm \eta\sim\mD_\eta} \Paren{\Abs{\bm \eta_i}\leq 1}=\alpha$,
  \end{enumerate}
  such that for every estimator $\hat{\beta}\,:\R^n\rightarrow\R^d$, given as input $X$ and $\bm y = X \bm\beta^* + \bm\eta$ with $\bm\beta^*\sim\mD_{\beta}$ and $\bm\eta\sim\mD_{\eta}$ sampled independently, one has
  \begin{align*}
    \E_{} \Normt{\hat{\beta}(\bm y)-\bm \beta^*}^2\geq \gamma\,.
  \end{align*}
\end{theorem}

A more precise version of \cref{thm:main-it-concise} is given by \cref{thm:main-it}. It states that there exists a natural distribution $\mD_X$ over $\R^{n\times d}$ such that \whp (i) ${\bm X\sim\mD_X}$ is $\max\{\Omega\paren{\frac{\log d}{\alpha^2}}\,, \Omega\paren{\frac{d}{\alpha}}\}$-spread, and (ii) given a matrix $X$ sampled from $\mD_X$ with $\max\{\Omega\paren{\frac{\log d}{\alpha^2}}\,, \Omega\paren{\frac{d}{\alpha}}\}$-spreadness as input for \cref{problem:oblivious-regression}, no estimator can obtain bounded error guarantees, for \textit{any} number of observations.
We remark that, in our construction we utilize the condition~${\transpose{X}X = n\cdot\Id}$ only to make the squared parameter error $\norm{\hat{\bt}-\bt^*}_2^2$ and squared prediction error $\frac{1}{n}\norm{X(\hat{\bt}-\bt^*)}_2^2$ equivalent.
Thus, \cref{thm:main-it-concise} immediately yields a lower bound for the prediction error as well.
This shows there is a fundamental difference between the statistical hardness of oblivious regression and its classical counterpart (with sub-gaussian noise): a statistical price to pay for robustness against oblivious adversaries.

It is fascinating to notice that, on one hand, current algorithms \cite{d2021consistent} obtain non-trivial error guarantees only for $\Omega(d/\alpha^2)$-spread design matrices; on the other hand, although those hard oblivious regression instances constructed in \cref{thm:main-it-concise} defy any non-trivial error guarantees, they do not rule out the existence of consistent estimators for other families of $o(d/\alpha^2)$-spread design matrices.
Thus, there remains a family of design matrices for which it is not known whether consistent oblivious regression can be achieved, and if so whether it can be done efficiently. This remains a pressing open question.

\paragraph{Certifying well-spreadness} 
As the well-spreadness of design matrices can guarantee efficient recovery in oblivious regression, it is natural to ask whether one can efficiently \textit{certify} the well-spreadness of a matrix. Similar questions have indeed been  investigated for RIP, due to its application in compressed sensing \cite{bandeira2013certifying, tillmann2014computational, NatarajanW14, koiran2014hidden, wang2016average, Weed18, ding2021average}. 
Unsurprisingly, certifying well-spreadness turns out to be NP-hard in the worst case (see \cref{thm:decide-spreadness-is-np-hard} for a proof). On the other hand, in the context of average-case analysis, there exists a regime where efficient algorithms can certify well-spreadness.

\begin{theorem}[Algorithms for certifying well-spreadness] \label[theorem]{theorem:main-algorithm}
  Fix arbitrary constants $\delta\in(0,1)$ and $C>0$.
  Let $\bm X\sim\Gau{0}{1}^{n\times d}$ with $n\geq Cd^2$. There exist a polynomial-time algorithm and a constant $C'=C'(C,\delta)\in(0,1)$ such that
  \begin{enumerate}
    \item $\bm X$ is $\bB{C'n,\delta}$-spread with probability $1-o(1)$;
    \item if $\bm X$ is not $\bB{C'n,\delta}$-spread, the algorithm outputs NO;
    \item if $\bm X$ is $\bB{C'n,\delta}$-spread, the algorithm outputs YES with probability $1-o(1)$.
  \end{enumerate}
\end{theorem}

We want to emphasize that, under the assumptions of \cref{theorem:main-algorithm}, (i) an $n$-by-$d$ Gaussian random matrix is $\LB{n}$-spread \whp;
(ii) if the sampled matrix is indeed $\LB{n}$-spread, the algorithm can efficiently certify this fact \whp;
(iii) the algorithm \emph{never} outputs false positives and thus guarantees the serviceability of the sampled matrix as a design matrix for oblivious regression.
\cref{theorem:main-algorithm} also directly implies that, in the regime~$n\ge \higherorder{d/\alpha^2}$ where there exist efficient algorithms for consistent oblivious regression, we may certify the well-spreadness of random design matrices, given $\LB{d^2}$ samples.

It is tempting to ask whether a similar verification algorithm can be designed with fewer observations $n\leq o(d^2)\,.$
However, we provide evidence of the computational hardness of this problem in the form of a lower bound against  low degree polynomials. This computational model captures state-of-the-art algorithms for many average-case problems such as sparse PCA,  tensor PCA or community detection (e.g. see \cite{hopkins2017bayesian, hopkins2018statistical, kunisky2019notes, ding2019subexponential, d2020sparse, choo2021complexity}).

\begin{theorem}[Lower bounds against low-degree polynomials] \label[theorem]{theorem:main-lower-bound}
  Let $t\leq O(\log n)^C$ for some arbitrary constant~$C>1$. 
  Let $\alpha=\alpha(n)\in (0,1)$.
  For\footnote{We use the notation $a\ll b$ for inequalities of the form $a\leq O(b/\polylog(b))$. The assumption~$\ap \gg d^{-1/2}$ derives from $d/\alpha^2 \ll (d/\alpha)^{4/3}$.} any $\ap \gg d^{-1/2}$ and $d/\alpha^2 \ll n\ll (d/\alpha)^{4/3}$, there exist two distributions over $n$-by-$d$ matrices, $\mD_0$ and $\mD_1$, such that
  \begin{enumerate}
    \item $\mD_0$ is the standard Gaussian distribution;
    \item $\bm X\sim \mD_0$ is $\LB{d/\alpha^2}$-spread with probability $1-o(1)$;
    \item $\bm X\sim \mD_1$ is not $\LB{d/\alpha^2}$-spread with probability $1-o(1)$;
    \item the two distributions are indistinguishable  with respect to all polynomials $p:\R^{n\times d}\rightarrow \R$ of degree at most $t$ in the sense that:
    \begin{align*}
      \frac{\E_{\mD_0}p(\bm X)-\E_{\mD_1}p(\bm X)}{\sqrt{\Var_{\mD_0}p(\bm X)}}\leq O(1)\,.
    \end{align*}
  \end{enumerate}
\end{theorem}

In other words, \cref{theorem:main-lower-bound} shows that low-degree polynomials cannot be used to distinguish between $\mD_0$  and $\mD_1$ as typical values of such polynomials look the same (up to a small difference) under both distributions.\footnote{See \cref{sec:low-deg-likelihood-ratio} for a more in-depth discussion concerning the low-degree likelihood ratio.}
An immediate consequence of this result is that there exists a family of $\LB{d/\alpha^2}$-spread matrices $X\in \R^{n\times d}$ with $d/\ap^2 \ll n\ll (d/\alpha)^{4/3}$, for which consistent oblivious regression is possible, but verifying whether the design matrix satisfy the required well-spread condition  is hard. 

We remark that there is no gap between the algorithmic result in \cref{theorem:main-algorithm} and the lower bound in \cref{theorem:main-lower-bound}, since \cref{theorem:main-lower-bound} is a corollary of \cref{thm:low-deg-hardness-certify-WS} wihch provides evidence of computational hardness for the entire regime $n\ll d^2$.

\subsection{Organization}

The rest of the paper is organized as follows. We present the high level ideas behind our results in \cref{section:techniques}. We prove \cref{thm:main-it-concise} in \cref{section:it-bounds}.
We obtain \cref{theorem:main-algorithm} and \cref{theorem:main-lower-bound} in \cref{section:cc-bounds}. We show NP-hardness of well-spreadness certification in \cref{sec:NP-hardness-deciding-well-spreadness}. Finally, necessary background notions can be found in \cref{section:preliminaries}. 

\section{Techniques}\label{section:techniques}

We present here the main ideas behind our results.

\subsection{Statistical lower bounds for regression}

Recall the linear model in \cref{problem:oblivious-regression},
\[
  \bm y = X\beta^*+\bm \eta,
\]
where we observe (a realization of) the random vector $\bm y$, the matrix $X\in \R^{n\times d}$ is a known fixed design, the vector $\beta^*\in \R^d$ is the unknown parameter of interest, and the noise vector $\bm \eta$ has independent, symmetrically distributed coordinates with $\min_{i\in[n]} \bbP\Set{\Abs{\bm \eta_i}\leq 1}=\alpha$. We will restrict our discussion to matrices satisfying $ \transpose{X}X= n\cdot \Id$, so that ---up to scaling--- there is no difference between prediction and parameter error.

To obtain an information-theoretic lower bound, we cast the problem as a distinguishing problem among $\ell$ hypotheses of the form:
\begin{align}\label{eq:techniques-hypothesis}
  H_i: \quad \bm y = X\beta_i + \bm \eta\,,
\end{align}
where we ought to make the vectors $\beta_1,\ldots, \beta_\ell\in \R^d$ as far as possible from each other.
It is remarkably easy to see that a small degree of spreadness is necessary to obtain any error guarantee in oblivious regression \cite{d2021consistentICML}. Let $X\in \R^{n\times d}$ be $o(1/\alpha)$-spread, then there exists a  vector $\beta\in \R^d$ and a set $S$ of cardinality $n-O(1/\alpha)$ such that $\norm{X_S\beta}_2 \leq o(1) \cdot \norm{X\beta}_2$.\footnote{$X_S\in \R^{n\times d}$ is the matrix obtained from $X$ by zeroing rows with index in $[n]\setminus S$.}
The problem with such a design matrix is that with probability $\Omega(1)$ all nonzero entries in $(X-X_S)\beta$ will be corrupted by (possibly unbounded) noise. As a result no estimator can provide guarantees of the form $\E\norm{\hat{\beta}(\bm y)-\beta^*}_2^2 \leq \gamma$ for any $\gamma>0$. In other words,  approximate recovery of the parameter vector is impossible.

Going beyond this $o(1/\alpha)$ barrier, however, turns out to be non-trivial. 
One issue with the above reasoning is that \textit{even if we knew} the uncorrupted entries, we would not be able to recover the hidden vector (with any bounded error), as no such entry contains information over $\beta^*$. In contrast, for any $X$ that is $m$-spread, with $m\geq \Omega(1/\alpha)$, if we knew the uncorrupted  entries, after filtering out the corrupted ones, the classical least squares estimator would yield error guarantees 
\[
  \E \Norm{\hat{\beta}(\bm y)-\beta^*}_2^2 \leq \UB{\frac{d}{m\ap}} \,.
\]
That is, knowing the uncorrupted entries one could achieve constant error for $\LB{d/\alpha}$-spread design matrices. Notice \cref{thm:main-it-concise} implies that, there exist oblivious regression instances where no estimator can achieve these guarantees.

We overcome this barrier with a construction  consisting of two main ingredients:
\begin{enumerate}
  \item An $m$-spread matrix $X$ and a set of vectors $\beta$ in $\R^d$ such that
  \begin{align}\label{eq:techniques-sparse-vector}
    \Norm{X_S\beta}_2 \leq o(1)\cdot \Norm{X \beta}_2
  \end{align}
  for some $S\subseteq [n]$ with $\card{S}\geq n-m$, and the subspace spanned by these vectors is high dimensional. 
  That is, a matrix whose column span contains many nearly orthogonal sparse vectors.
  \item A distribution $D_{\eta}$ over $\R$, for the entries of $\bm \eta$, satisfying the constraints in \cref{problem:oblivious-regression}, and with the additional properties:
  \begin{itemize}
    \item \textit{Low shift-sensitivity:} the distribution looks approximately the same after an additive shift in the following sense. If $D_\eta(k)$ is the distribution shifted by $k$, then the Kullback-Leibler divergence $\kldiv{D_\eta}{D_\eta(k)}$ is small.
    \item \textit{Insensitivity to scaling:} the Kullback-Leibler divergence does not change significantly upon scaling in the sense that $\kldiv{D_\eta}{D_\eta(k)} \approx \kldiv{\rho \cdot D_\eta}{\rho\cdot D_\eta(k)}$, for any $\rho >0$.
  \end{itemize}
\end{enumerate}
Sparsity of the noiseless observation vectors $X\beta_i$ as in \cref{eq:techniques-sparse-vector}, combined with low shift-sensitivity of the noise distribution will allow us to make the different hypotheses indistinguishable. Then insensitivity to scaling will make the prediction error arbitrarily large.

\paragraph{Noise distributions with low sensitivity}  
It turns out that constructing a distribution with low sensitivity is straightforward. We consider the symmetric geometric distribution $\sgd{c, \lambda}$ with probability mass function
\begin{equation*} 
  p(k) = 
  \begin{cases}
    \ap, & k=c \\
    \frac{1-\ap}{2} \cdot \lambda(1-\lambda)^{|k|-1}, & k=c\pm1, c\pm2, c\pm3,...
  \end{cases}
\end{equation*}
Clearly, $\rho\cdot\sgd{0, \lambda}$ is symmetric and satisfies $\bbP_{\bm z \sim \rho\cdot\sgd{0, \lambda}}\Paren{\abs{\bm z}\leq 1} = \alpha$ for any $\rho>1$ and any $\lambda\in(0,1)$. Moreover, as $\kldiv{\sgd{0, \lambda}}{\sgd{c, \lambda}} = \kldiv{\rho\cdot\sgd{0, \lambda}}{\rho \cdot\sgd{c, \lambda}}$, we have the desired insensitivity to scaling. Finally, for small enough values of $\lambda$ low shift-sensitivity holds  for integer-valued shifts, as the distribution is discrete in nature. (See \cref{lemma:facts-sym-geo-distn} for a formal statement.)

\paragraph{Matrices spanning sparse subspaces} 
To construct the aforementioned design matrix, we wish to find two rectangular matrices $A\in \R^{m\times d}$ and $B\in \R^{(n-m)\times d}$ such that $A$ is $\Omega(m)$-spread and the row spans of $A$ and $B$ are orthogonal to each other.
Let $X\in\R^{n\times d}$ be the matrix obtained by stacking $A$ onto $B$. 
Then for any vector $\bt$ in the row span of $A$, i.e. $\beta \in \mathrm{rspan}(A)$, we have
\begin{align*}
  \Snorm{X\beta}_2 = 
  \Snorm{\begin{bmatrix} A \\ B\end{bmatrix}\beta}_2 = 
  \Snorm{A\beta}_2 + \Snorm{B\beta}_2 = 
  \Snorm{A\beta}_2\,,
\end{align*}
and thus $X$ also has the required $m$-spreadness.
To find two such matrices $A$ and $B$, the following observation turns out to be crucial:\footnote{We remark that a similar observation holds for other distributions (e.g. Gaussian). See \cref{thm:sub-Gaussian-matrix-WS} for a formal proof. The value of integer values will become evident in the interplay between the design matrix and the noise distribution.} 
\begin{quote}
  \begin{center}
    \textit{An $m$-by-$d$ Rademacher matrix is $\Omega(m)$-spread with high probability.}
  \end{center}
\end{quote}
With this ingredient we are now ready to construct the design matrix $X$.
Let $\mR$ denote Rademacher distribution.
Let $\bm A^*\sim \mR^{m\times d}$ and $\bm B^*\sim \mR^{(n-m)\times d}$ be independently sampled, and let $V\subseteq \R^d$ be an $\Omega(d)$-dimensional subspace. We construct
\begin{align*}
  \bm X= \begin{bmatrix}
    \bm A\\
    \bm B
  \end{bmatrix}=
\begin{bmatrix}
    \bm A^* \Pi_V\\
    \bm B^* \Pi_{V^\bot}
  \end{bmatrix}\,,
\end{align*}
where $\Pi_V$ denotes the projector onto the subspace $V$. Then the row spans of $\bm A\,, \bm B$ are orthogonal.

\paragraph{Putting things together}
The  ideas presented above allow us to construct hypotheses as in \cref{eq:techniques-hypothesis} that are indistinguishable from each other even though the corresponding parameter vectors are far from each other. 
Let $\beta\,, \beta'\in \R^d$ be distinct vectors in the row span of $\bm A$ with integer coordinates satisfying $\Norm{\beta-\beta'}_2\geq \LB{\sqrt{d}}$ and $\Norm{\beta}_2, \Norm{\beta'}_2\leq \UB{\sqrt{d}}$. 
By construction, $\bm X \beta$ and $\bm X \beta'$ are both $n$-dimensional vectors with $\Omega(m)$ nonzero entries, each integer-valued. 
So, by low shift-sensitivity of the noise distribution,
\begin{align*}
  H: \quad \bm y &= \bm X\beta + \bm \eta\,,\\
  H': \quad \bm y &= \bm X\beta' + \bm \eta\,
\end{align*}
are stastistically indistinguishable. Finally, expanding on these ideas \cref{thm:main-it-concise} will follow. Furthermore, by insensitivity to scaling of $D_\eta$ we can now blow up the error by scaling up  $\sigma\cdot \bm y = \bm X (\sigma\cdot\beta) +\sigma\cdot \bm \eta$ and $\sigma\cdot \bm y' = \bm X (\sigma\cdot\beta') +\sigma\cdot \bm \eta$, for any $\sigma > 0$, without making the distinguishing problem easier.

\subsection{Differences between well-spreadness and RIP, RE properties}\label{section:differences-well-spreadness-others}

In the context of compressed sensing we are given $n\leq d$ observations of the form $\bm y_i = \iprod{M_i, \beta}+\bm \eta_i$ with $M_i, \beta\in \R^d$ and $\bm \eta$ being additive noise.  In order to guarantee recovery of the compressed vector $\beta$, RIP \cite{candes2005decoding, donoho2006compressed, candes2006stable, kashin2007remark} is arguably the most popular condition to enforce on the sensing matrix:

\begin{definition}[Restricted isometry property]
  We say a matrix $M\in \R^{n\times d}$ satisfies the $(k, \delta)$-restricted isometry property (RIP) if
  \begin{align*}
    (1-\delta)\normt{v}^2\leq \normt{Mv}^2\leq (1+\delta)\normt{v}^2
  \end{align*}
  for every vector $v$ with at most $k$ nonzero entries.
\end{definition}

We argue here that the relation between well-spreadness and oblivious regression fundamentally differs from that of RIP and compressed sensing in two ways.

First, while state-of-the-art algorithms for compressed sensing rely on RIP in order to filter out the noise in the observations  and recover the hidden vector, it is known that small prediction error can be achieved in exponential time \textit{without} any constraint on the sensing matrix \cite{bunea2007sparsity}. 
In contrast, \cref{thm:main-it-concise} shows that in the context of oblivious regression, no algorithm can achieve even small prediction error for a family of design matrices that are not sufficiently well-spread.

Second, RIP is not purely a condition of the column span of the sensing matrix. In particular, if $M$ satisfies $(k, \delta)$-RIP, then the kernel of $M$ must be $\Omega(k)$-spread \cite{spread-sparse-matrices}.\footnote{The  careful reader may have noticed that $\delta$ seems to play no role in this implication. In fact, the relation is more general than what we consider here. See \cite{spread-sparse-matrices}.} Conversely, it is easy to construct matrices with well-spread column span and kernel containing sparse vectors.

As the following examples show, it is easy to construct matrices that satisfy RIP but are not well-spread and vice versa.

\begin{example}[RIP but not even $1$-spread] 
  Let $\bm W\sim \Gau{0}{1}^{(n-1)\times (d-1)}$ and consider the following $n$-by-$d$ matrix (we do not fix the relation between $n$ and $d$),
  \begin{align*}
    \bm M = 
    \begin{bmatrix}
      1 & 0\\
      0 & \frac{1}{\sqrt{n-1}} \bm W
    \end{bmatrix}.
  \end{align*}
  If $n\gtrsim\delta^{-2}k\log{d}$, then with high probability, $\bm M$ satisfies $(k,\delta)$-RIP.
  However, $\bm M$ is not even $1$-spread, since its column span contains the canonical basis vector $e_1\in\R^n$.
\end{example}

\begin{example}[Well-spread but not satisfying RIP]
  Let $\bm W\sim \Gau{0}{1}^{n\times (d-1)}$ and consider the following $n$-by-$d$ matrix,
  \begin{align*}
    \bm M = 
    \begin{bmatrix}
      v & \frac{1}{\sqrt{n}} \bm W
    \end{bmatrix},
  \end{align*}
  where $v\in\R^n$ is a unit vector parallel or highly correlated to the first column of $\bm W$.
  Then $\bm M$ cannot satisfy RIP or RE.
  However, if $n \geq Cd$ for some sufficiently large absolute constant $C$, then it is easy to verify that, $\bm M$ is $\LB{n}$-spread \whp.
\end{example}

\section{Background} \label{section:preliminaries}

\subsection{Basic notation} \label{sec:basic-notation}
We use the convention $\N = \Set{0,1,2,3,...}$.
For a positive integer $n$, let $[n] := \Set{1,2,...,n}$.
For $\ap\in\N^n$, define $|\ap| := \sum_{i=1}^{n}\ap_i$.
For a vector $v\in\R^n$, let $\supp(v) := \Set{i\in[n] : v_i\neq0}$ be its support, $\Norm{v}_p := \bB{\sum_{i=1}^{n}|v_i|^p}^{1/p}$ be its $\ell_p$-norm ($p\geq1$), and $\Norm{v}_0 := \abs{\supp(v)}$.
Given a vector $v\in\R^n$ and a subset $S\subseteq[n]$, let $v_S\in\R^{|S|}$ denote the projection of $v$ onto the coordinates in $S$.
For a matrix~$X$, let $\col{X}$ denote its column span, $\mathrm{rspan}(A)$ denote its row span, and $\ker{X}$ denote its kernel or null space. Let $\sigma_{\min}(X)$ and $\sigma_{\max}(X)$ denote its minimum and maximum singular values respectively.
We use standard asymptotic notations $\LB{\cdot}, \UB{\cdot}, \lesssim, \gtrsim$ to hide absolute multiplicative constants. 
Throughout this paper, we write random variables in boldface.
We say an event happens \emph{with high probability} if it happens with probability $1-o(1)$.
Given two distributions $\nu$ and $\mu$, let $\KL{\nu}{\mu}$ denote their Kullback-Leibler divergence. For two random variables $\bm X$ and $\bm Y$, we write $\KL{\bm X}{\bm Y}$ to denote the Kullback-Leibler divergence between their distributions.
By saying Gaussian (or Rademacher) matrix, we mean a matrix whose entries are independent standard Gaussian (or Rademacher) random variables.
Unless explicitly stated, the base of logarithm is the natural number $e$.

\subsection{Fano's method} \label{sec:fano-method}

Fano's method is a classical approach to proving lower bounds for statistical estimation problems, which we apply to prove the information-theoretic lower bounds in \cref{section:it-bounds}.

Suppose we are given an $\delta$-separated set $\mB\subset\R^d$. That is, $\Norm{\bt-\bt'}_2\geq\delta$ for any distinct $\bt,\bt'\in\mB$.
Let $D_\beta$ be the uniform distribution over $\mB$ and $\bm\beta^* \sim D_\beta$. 
Let $X\in\mat{n}{d}$ be a known design matrix and $\bm\eta$ be the noise vector.
Observing $\bm y = X\bm\bt^* + \bm\eta$, the hypothesis testing problem is to distinguish $|\mB|$ distributions $\Set{X\bt+\bm\eta : \bt\in\mB}$.
Let $\hat{\bt}:\R^n\to\R^d$ be an arbitrary estimator for the linear regression problem.
By a reduction from the hypothesis testing problem and applying Fano's inequality (\cref{lem:Fano-ineq}) combined with the convexity of the Kullback-Leibler divergence, one has\footnote{We refer interested readers to \cite{scarlett2019introductory} for a proof of \cref{eq:Fano-minimax-lower-bound-real} as well as more applications of Fano's method.} 
\begin{equation} \label{eq:Fano-minimax-lower-bound-real}
  \E \Norm{\hat{\bt}(\bm y)-\bm\bt^*}_2^2 
  \geq \frac{\delta^2}{4} \bB{1-\frac{\max\limits_{\bt,\bt'\in\mB}\KL{X\bt+\bm\eta}{X\bt'+\bm\eta}+\log2}{\log|\mB|}}.
\end{equation}


\begin{lemma}[Fano's inequality] \label[lemma]{lem:Fano-ineq}
  Let $\Sigma$ be a finite set and $\bm J$ be a random variable uniformly distributed over $\Sigma$.
  Suppose $\bm J \to \bm Z \to \hat{\bm J}$ is a Markov chain.
  Then,
  \[ \prob{\bm J \neq \hat{\bm J}} \geq 1 - \frac{I(\bm J;\bm Z)+\log2}{\log|\Sigma|}, \]
  where $I(\bm J;\bm Z)$ denotes the mutual information between $\bm J$ and $\bm Z$.
\end{lemma}

\subsection{Spreadness and distortion} \label{sec:spreadness-and-distortion}

\begin{definition}[$\ell_p$-spreadness] \label[definition]{def:spreadness}
  Let $p\geq1$, $\delta\in[0,1]$, $n\in\N$, and $m \leq n$.
  A vector $v\in\R^n$ is said to be $(m,\delta)$-$\ell_p$-spread if for every subset $S\subseteq[n]$ with $|S| \leq m$, we have
  \[ \Norm{v_S}_p \leq \delta \cdot \Norm{v}_p. \]
  A subspace $V\subseteq\R^n$ is said to be $(m,\delta)$-$\ell_p$-spread if every vector $v\in V$ is $(m,\delta)$-$\ell_p$-spread.
  A matrix is said to be $(m,\delta)$-$\ell_p$-spread if its column span $\col{X}$ is $(m,\delta)$-$\ell_p$-spread.
\end{definition}

When the ambient dimension $n$ is clear from the context, there are three parameters, i.e. $p,m,\delta$, to be specified in \cref{def:spreadness}.
If the value of $p$ is not specified, set $p=2$ by default. That is, $(m,\delta)$-spreadness means $(m,\delta)$-$\ell_2$-spreadness.
In certain cases (e.g. oblivious linear regression), we are more interested in capturing the dependence of $m$ than the dependence of $\delta$ on other paramters (e.g. the ambient dimension $n$). 
Then it is more convenient to hide $\delta$ as long as it is $\LB{1}$.
Concretly, we say a vector, or a subspace, or a matrix is $m$-$\ell_p$-spread if there exists a absolute constant $c\in(0,1)$ such that it is $\bB{m,c}$-$\ell_p$-spread.

We introduce the following definition that is closely related to spreadness and has important algorithmic implications in \cref{section:certifying-well-spreadness}.

\begin{definition}[$\ell_p$-vs-$\ell_q$ distortion] \label[definition]{def:q/p-distortion}
  Given $1\leq p<q$, the $\ell_p$-vs-$\ell_q$ distortion of a nonzero vector $v\in\R^n$ is defined by
  \[ \Delta_{p,q}(v) := \frac{\Norm{v}_q}{\Norm{v}_p} \cdot n^{\frac{1}{p}-\frac{1}{q}}. \]
  The $\ell_p$-vs-$\ell_q$ distortion of a subspace $V\subseteq\R^n$ is defined by
  \[ \Delta_{p,q}(V) := \max_{v\in V, v\neq0} \Delta_{p,q}(v). \]
  The $\ell_p$-vs-$\ell_q$ distortion of a matrix $X$ is defined by
  \[ \Delta_{p,q}(X) := \Delta_{p,q}(\col{X}). \]
\end{definition}

By \Hoelder's inequality and monotonicity of $\ell_p$ norm, it is easy to check $1 \leq \Delta_{p,q}(v) \leq n^{\frac{1}{p}-\frac{1}{q}}$ for any nonzero vector $v\in\R^n$. 
Note that for a nonzero vector $v\in\R^n$, $\Delta_{p,q}(v)=1$ if and only if $|v_1|=\cdots=|v_n|$; $\Delta_{p,q}(v)=n^{\frac{1}{p}-\frac{1}{q}}$ if and only if $\Norm{v}_0=1$.
Intuitively, low distortion implies well-spreadness, which is formalized in the following proposition.

\begin{proposition} \label[proposition]{prop:low-distortion-implies-well-spreadness}
  Let $1\leq p<q$ and $V$ be a subspace of $\R^n$.
  \begin{enumerate}
    \item If $\Delta_{p,q}(V) \leq \Delta$, then $V$ is $\bB{m, \delta_p}$-$\ell_p$-spread with
    \[ \delta_p = (m/n)^{\frac{1}{p}-\frac{1}{q}}\Delta. \]
    \item If $V$ is not $\bB{m,\delta}$-$\ell_p$-spread, then 
    \[ \Delta_{p,q}(V) > \delta(n/m)^{\frac{1}{p}-\frac{1}{q}}. \]
    \item If $\Delta_{p,q}(V) \leq \Delta$, then $V$ is $\bB{m, \delta_q}$-$\ell_q$-spread with
    \[ \delta_q^q = 1-\bB{\Delta^{-p}-(m/n)^p}^{\frac{q}{p}}. \]
  \end{enumerate}  
\end{proposition}
\begin{proof} ~
\begin{enumerate}
\item 
  By H{\"o}lder's inequality, for any nonzero vector $x$,
  \[ \abs{\supp{x}}^{\frac{1}{q}-\frac{1}{p}} \leq \frac{\Norm{x}_q}{\Norm{x}_p} \leq 1. \]
  For any nonzero vector $x\in V$ and any subset $S\subseteq[n]$ with $|S|\leq m$, we have
  \begin{align*}
    \Norm{x_S}_p 
    \leq |S|^{\frac{1}{p}-\frac{1}{q}} \cdot \Norm{x_S}_q
    \leq |S|^{\frac{1}{p}-\frac{1}{q}} \cdot \Norm{x}_q
    \leq \Delta (m/n)^{\frac{1}{p}-\frac{1}{q}} \Norm{x}_p.
  \end{align*}
\item 
  Since $V$ is not $\bB{m,\delta}$-$\ell_p$-spread, then by definition there exist $x\in V$ with $\Norm{x}_p=1$ and $S\subset[n]$ with $|S|\leq m$ such that 
  \[ \Norm{x_S}_p > \delta\Norm{x}_p = \delta. \]
  Applying \Hoelder's inequality,
  \[ \Norm{x}_q \geq \Norm{x_S}_q \geq |S|^{\frac{1}{q}-\frac{1}{p}} \Norm{x_S}_p > \delta|S|^{\frac{1}{q}-\frac{1}{p}}. \]
  Then, 
  \[ \Delta_{p,q}(V) \geq \Delta_{p,q}(x) 
  = n^{\frac{1}{p}-\frac{1}{q}} \frac{\Norm{x}_q}{\Norm{x}_p}
  > \delta (n/m)^{\frac{1}{p}-\frac{1}{q}}. \]
\item 
  Fix an arbitrary vector $x\in V$ with $\Norm{x}_q=1$. 
  Then for any subset $S\subseteq[n]$ with $|S|\leq m$, we have
  \begin{align*}
    \Norm{x_S}_p 
    \leq |S|^{\frac{1}{p}-\frac{1}{q}} \cdot \Norm{x_S}_q
    \leq m^{\frac{1}{p}-\frac{1}{q}}.
  \end{align*}
  As $\Norm{x}_p \geq \Delta^{-1} n^{\frac{1}{p}-\frac{1}{q}}$, then 
  \[ \Norm{x_{\bar{S}}}_p 
  = \bB{\Norm{x}_p^p-\Norm{x_S}_p^p}^{\frac{1}{p}}
  \geq \bB{\Delta^{-p}-(m/n)^p}^{\frac{1}{p}} n^{\frac{1}{p}-\frac{1}{q}}. \]
  Applying H{\"o}lder's inequality again,
  \begin{align*}
    \Norm{x_{\bar{S}}}_q
    \geq \abs{\bar{S}}^{\frac{1}{q}-\frac{1}{p}} \Norm{x_{\bar{S}}}_p
    \geq \bB{\Delta^{-p}-(m/n)^p}^{\frac{1}{p}}.
  \end{align*}
  Thus,
  \begin{align*}
    \Norm{x_S}_q^q 
    = \Norm{x}_q^q-\Norm{x_{\bar{S}}}_q^q
    \leq 1 - \bB{\Delta^{-p}-(m/n)^p}^{\frac{q}{p}}.
  \end{align*}
\end{enumerate}
\end{proof}

In particular, given $1\leq p<q$ and a subspace $V\subseteq\R^n$, if $\Delta_{p,q}(V) \leq \UB{1}$, then $V$ is both $\LB{n}$-$\ell_p$-spread and $\LB{n}$-$\ell_q$-spread.
On the other hand, if $V$ is not $\LB{n}$-$\ell_p$-spread, then $\Delta_{p,q}(V)\geq\higherorder{1}$.

\subsection{Low-degree likelihood ratio} \label{sec:low-deg-likelihood-ratio}
To better understand the hardness result in \cref{section:avg-case-hardness-well-spreadness}, we briefly introduce the \emph{low-degree polynomial method} \cite{hopkins2018statistical} that is developed for studying computational complexity of high-dimensional statistical inference problems.
For further details about the low-degree polynomial method, we refer interested readers to \cite{kunisky2019notes}.

Consider in an asymptotic regime ($N\to\infty$) the hypothesis testing problem of distinguishing two sequences of hypotheses $\mu=\Set{\mu_N}_{N\in\N}$ and $\nu=\Set{\nu_N}_{N\in\N}$, where $\mu_N$ and $\nu_N$ are probability distributions over $\R^N$.
We are interested in the case where $\nu$, the \emph{null distribution}, contains pure noise (e.g. $\nu_N = \Gau{0}{1}^{N}$), and $\mu$, the \emph{planted distribution}, contains planted signal.
A sequence of test functions $f=\Set{f_N}_{N\in\N}$ with $f_N:\R^N\to\Set{0,1}$ is said to \emph{strongly distinguish} $\mu$ and $\nu$ if 
\begin{equation} \label{eq:strong-distinguishability}
  \lim_{N\to\infty} \Pr_{\mu}\bB{f_N(\bm X)=1} = 1 \he 
  \lim_{N\to\infty} \Pr_{\nu}\bB{f_N(\bm X)=0} = 1.
\end{equation}

In other words, strong distinguishability means both type I and type II errors go to $0$ as $N\to\infty$.
We only consider the case where $\mu$ is absolutely continuous with respect to $\nu$. 
The \emph{likelihood ratio} defined by
\begin{equation*} 
  L(X) := \frac{\dif \mu}{\dif \nu}(X)
\end{equation*}
is an optimal test function in the following sense.

\begin{proposition} \label[proposition]{prop:optimality-likelihood-ratio}
  Suppose $\mu$ is absolutely continuous with respect to $\nu$. 
  The unique solution of the optimization problem
  \[ \max \E_{\mu}\sB{f(\bm X)} \quad \text{subject to } \E_{\nu}\sB{f(\bm X)^2}=1 \]
  is $L(X)/\sqrt{\E_{\nu}\sB{L(\bm X)^2}}$ and the value of the optimization problem is $\sqrt{\E_{\nu}\sB{L(\bm X)^2}}$.
\end{proposition}

Furthermore, classical decision theory tells us $\E_{\nu}\sB{L(\bm X)^2}$ characterizes strong distinguishability in the following way.

\begin{proposition} \label[proposition]{prop:2nd-moment-method}
  If $\E_{\nu}\sB{L(\bm X)^2}$ remains bounded as $N\to\infty$, then $\mu$ and $\nu$ is not strongly distinguishable in the sense of \cref{eq:strong-distinguishability}.
\end{proposition}

One limitation of the above classical decision theory is that no computational-complexity considerations are involved.
With the goal of studying whether a hypothesis testing problem is strongly distinguishable computation-efficiently, the low-degree polynomial method uses low-degree multivariate polynomials in the entries of $\bm X$ sampled from either $\mu$ or $\nu$ as a proxy for efficiently-computable functions.

\begin{definition}[Low-degree likelihood ratio] \label[definition]{def:low-deg-likelihood-ratio}
  The degree-$D$ likelihood ratio, denoted by $L^{\leq D}$, is the orthogonal projection\footnote{We consider the Hilbert space endowed with inner product $\inner{f}{g}:=\E_{\nu}\sB{f(\bm X)g(\bm X)}$.} of the likelihood ratio $L=\dif\mu/\dif\nu$ onto the subspace of polynomials of degree at most $D$.
\end{definition}

We have the following low-degree analogue of \cref{prop:optimality-likelihood-ratio}.

\begin{proposition} \label[proposition]{prop:optimality-low-deg-likelihood-ratio}
  Suppose $\mu$ is absolutely continuous with respect to $\nu$.
  The unique solution of the optimization problem
  \[ \max_{f\in\R[X]_{\leq D}} \E_{\mu}\sB{f(\bm X)} \quad \text{subject to } \E_{\nu}\sB{f(\bm X)^2}=1 \]
  is $L^{\leq D}(X)/\sqrt{\E_{\nu}\sB{L^{\leq D}(\bm X)^2}}$ and the value of the optimization problem is $\sqrt{\E_{\nu}\sB{L^{\leq D}(\bm X)^2}}$.
\end{proposition}

The following informal conjecture \cite[Conjecture 1.16]{kunisky2019notes}, which itself is based on \cite[Conjecture 2.2.4]{hopkins2018statistical}, can be thought of as an computational analogue of \cref{prop:2nd-moment-method}.
\begin{conjecture}[Informal] \label{conj:low-deg}
  For ``sufficiently nice" sequences of probability distributions $\mu$ and $\nu$, if there exists $\ep>0$ and $D=D(N)\geq\bB{\log{N}}^{1+\ep}$ for which $\E_{\nu}\sB{L^{\leq D}(\bm X)^2}$ remains bounded as $N\to\infty$, then there is no polynomial-time algorithm that strongly distinguishes $\mu$ and $\nu$.
\end{conjecture}

\section{Information-theoretic bounds for oblivious regression} \label{section:it-bounds}

We state and prove the more precise and technical version of \cref{thm:main-it-concise} that shows, there exists a family of $\max \Set{\Omega\Paren{\frac{\log d}{\alpha^2}}\,, \Omega\Paren{\frac{d}{\alpha}}}$-spread design matrices such that, consistent estimation is information-theoretically impossible in oblivious linear regression.

\begin{theorem} \label[theorem]{thm:main-it}
  Let $\alpha = \alpha(n) \in (0,1)$.
  For arbitrary $\gamma=\gamma(n)>0$, there exist:
  \begin{enumerate}
    \item a distribution $\mD_X$ over $n\times d$ matrices $X$ with $X^\top X=n\cdot\Id$,
    \item a distribution $\mD_\beta$ over $d$-dimensional vectors, and
    \item a distribution $\mD_\eta$ ---independent of $\mD_X$ and $\mD_\beta$--- over $n$-dimensional vectors with independent, symmetrically distributed entries satisfying $\min_{i\in[n]}\bbP_{\bm \eta\sim\mD_\eta} \Paren{\Abs{\bm \eta_i}\leq 1}=\alpha$,
  \end{enumerate}
  such that,
  \begin{enumerate}
    \item $\bm X\sim\mD_X$ is $\max \Set{\Omega\Paren{\frac{\log d}{\alpha^2}}\,, \Omega\Paren{\frac{d}{\alpha}}}$-spread \whp; and
    \item for every estimator $\hat{\beta}\,:\R^n\rightarrow\R^d$, given as input $\bm X$ and $\bm y = \bm X \bm\beta^* + \bm\eta$ with $\bm X\sim\mD_X$, $\bm\eta\sim\mD_{\eta}$, and $\bm\beta^*\sim\mD_{\beta}$ sampled independently, one has
    \begin{align*}
      \E_{} \Normt{\hat{\beta}(\bm y)-\bm \beta^*}^2\geq \gamma\,,
    \end{align*}
    conditioning on $\bm X$ being $\max \Set{\Omega\Paren{\frac{\log d}{\alpha^2}}\,, \Omega\Paren{\frac{d}{\alpha}}}$-spread.
  \end{enumerate}
\end{theorem}

To prove \cref{thm:main-it}, we provide the following two lemmas which we will prove in \cref{sec:proof-of-d/alpha} and \cref{sec:proof-of-log{d}/alpha^2} respectively.

\cref{lem:lower-bound-d/alpha} shows that, there exists a family of $\LB{\frac{d}{\ap}}$-spread design matrices such that, consistent estimation is information-theoretically impossible in oblivious linear regression.

\begin{lemma} \label[lemma]{lem:lower-bound-d/alpha}
  Let $\alpha=\ap(n) \leq O(1)$. 
  For arbitrary $\gamma=\gamma(n)>0$, there exist:
  \begin{enumerate}
    \item a distribution $\mD_X$ over $n\times d$ matrices $X$ with $X^\top X=n\cdot\Id$,
    \item a distribution $\mD_\beta$ over $d$-dimensional vectors, and
    \item a distribution $\mD_\eta$ ---independent of $\mD_X$ and $\mD_\beta$--- over $n$-dimensional vectors with independent, symmetrically distributed entries satisfying $\min_{i\in[n]}\bbP_{\bm \eta\sim\mD_\eta} \Paren{\Abs{\bm \eta_i}\leq 1}=\alpha$,
  \end{enumerate}
  such that,
  \begin{enumerate}
    \item $\bm X\sim\mD_X$ is $\Omega\Paren{\frac{d}{\alpha}}$-spread \whp; and
    \item for every estimator $\hat{\beta}\,:\R^n\rightarrow\R^d$, given as input $\bm X$ and $\bm y = \bm X \bm\beta^* + \bm\eta$ with $\bm X\sim\mD_X$, $\bm\eta\sim\mD_{\eta}$, and $\bm\beta^*\sim\mD_{\beta}$ sampled independently, one has
    \begin{align*}
      \E_{} \Normt{\hat{\beta}(\bm y)-\bm \beta^*}^2\geq \gamma\,,
    \end{align*}
    conditioning on $\bm X$ being $\Omega\Paren{\frac{d}{\alpha}}$-spread.
  \end{enumerate}
\end{lemma}

\cref{lem:lower-bound-log{d}/alpha^2} shows that, there exists a family of $\LB{\frac{\log{d}}{\ap^2}}$-spread design matrices such that, consistent estimation is information-theoretically impossible in oblivious linear regression.

\begin{lemma} \label[lemma]{lem:lower-bound-log{d}/alpha^2}
  Let $\alpha=\ap(n) \leq O(1)$.
  For arbitrary $\gamma=\gamma(n)>0$, there exist:
  \begin{enumerate}
    \item a distribution $\mD_X$ over $n\times d$ matrices $X$ with $\Omega\Paren{\frac{\log d}{\alpha^2}}$-spreadness and $X^\top X=n\cdot\Id$,
    \item a distribution $\mD_\beta$ over $d$-dimensional vectors, and
    \item a distribution $\mD_\eta$ ---independent of $\mD_X$ and $\mD_\beta$--- over $n$-dimensional vectors with independent, symmetrically distributed entries satisfying $\min_{i\in[n]}\bbP_{\bm \eta\sim\mD_\eta} \Paren{\Abs{\bm \eta_i}\leq 1}=\alpha$,
  \end{enumerate}
  such that for every estimator $\hat{\beta}\,:\R^n\rightarrow\R^d$, given as input $\bm X$ and $\bm y = \bm X \bm\beta^* + \bm\eta$ with $\bm X\sim\mD_X$, $\bm\eta\sim\mD_{\eta}$, and $\bm\beta^*\sim\mD_{\beta}$ sampled independently, one has
  \begin{align*}
    \E_{} \Normt{\hat{\beta}(\bm y)-\bm \beta^*}^2\geq \gamma\,.
  \end{align*}
\end{lemma}

\cref{thm:main-it} follows directly from the above two lemmas.
\begin{proof}
  By \cref{lem:lower-bound-d/alpha} and \cref{lem:lower-bound-log{d}/alpha^2}.
\end{proof}

We introduce here the noise distribution (i.e. $D_\eta$) which will play a crucial role in our proof of \cref{lem:lower-bound-d/alpha} and \cref{lem:lower-bound-log{d}/alpha^2}.

\begin{definition}[Symmetric geometric distribution] \label[definition]{def:sym_geo_distn}
  The symmetric geometric distribution with location parameter $c\in\Z$ and scale parameter $\lambda\in(0,1)$, denoted by $\sgd{c, \lambda}$, is a discrete distribution supported on $\Z$. 
  Its probability mass function is defined as
  \begin{equation} \label{eq:sym_geo_distn}
    p(k) = 
    \begin{cases}
      \ap, & k=c, \\
      \frac{1-\ap}{2} \cdot \lambda(1-\lambda)^{|k|-1}, & k=c\pm1, c\pm2, c\pm3, \cdots ,
    \end{cases}
  \end{equation}
  where $\ap$ is the same $\ap$ in \cref{problem:oblivious-regression}. 
  Let $\mG(\lambda) = \mG(0,\lambda)$ by  default.
\end{definition}

We collect several useful facts about symmetric geometric distributions in the following lemma.

\begin{lemma} \label[lemma]{lemma:facts-sym-geo-distn}
  Let $\sgd{c,\lambda}$ be the symmetric geometric distribution with parameters $c$ and $\lambda$, as defined in \cref{def:sym_geo_distn}.
  \begin{enumerate}
    \item For any $\sigma>0$, $c\in\Z$, and $\lambda\in(0,1)$, we have
    \[ \KL{\sigma\cdot\mG(\lambda)}{\sigma\cdot\mG(c,\lambda)} =  \KL{\mG(\lambda)}{\mG(c,\lambda)}. \]
    
    \item Suppose $\ap\leq1/4$. Let $\lambda=2\ap$. Then,
    \[ \KL{\mG(\lambda)}{\mG(1,\lambda)} \leq 4\ap^2. \]
    
    \item Suppose $d\geq4$ and $\ap\leq1/2$.
    Let $\lambda=2\ap d^{-5}$. Then for any $\Delta\in[d^4]$, we have
    \[ \KL{\mG(\lambda)}{\mG(\Delta,\lambda)} \leq 8\ap\cdot\log{d}. \]
  \end{enumerate}
\end{lemma}

\begin{proof}
  Given $\lambda\in(0,1)$ and $\Delta\in\Z$, let $p$ and $q$ be the probability mass functions of $\mG(\lambda)$ and $\mG(\Delta,\lambda)$ respectively. 
  By definition,
  \[ 
    \KL{\mG(\lambda)}{\mG(\Delta,\lambda)} 
    = \sum_{k=-\infty}^{\infty} p(k)\log\frac{p(k)}{q(k)}
    = \underbrace{\sum_{k\neq0,\Delta} p(k)\log\frac{p(k)}{q(k)}}_{=:D(\lambda,\Delta)}  + \underbrace{\sum_{k=0,\Delta} p(k)\log\frac{p(k)}{q(k)}}_{=:D'(\lambda,\Delta)}.
  \]
  After some direct computations, we have
  \begin{align*}
    D(\lambda,\Delta) &= \frac{1-\ap}{2} \cdot \frac{1}{\lambda} \cdot \log\frac{1}{1-\lambda} \cdot \sB{2\lambda\Delta + 2(1-\lambda)^\Delta - 2 + \lambda^2\Delta(1-\lambda)^{\Delta-1}}, \he \\
    D'(\lambda,\Delta) &= \ap \cdot \bB{1 - \frac{(1-\ap)\lambda(1-\lambda)^{\Delta-1}}{2\ap}} \cdot \log\frac{2\ap}{(1-\ap)\lambda(1-\lambda)^{\Delta-1}}.  \numberthis \label{eq:compute-KL-div}
  \end{align*}
  We remark that both $D(\lambda,\Delta)$ and $D'(\lambda,\Delta)$ can be viewed as the Kullback-Leibler divergence between two probabilistic distributions up to a positive scaling factor.
  Thus, $D(\lambda,\Delta)$ and $D'(\lambda,\Delta)$ are always non-negative regardless of $\lambda$ and $\Delta$.
  
  \begin{enumerate}
  \item By definition.
  \item 
    Substituting $\lambda$ by $2\ap$ and $\Delta$ by 1 in \cref{eq:compute-KL-div}, we have
    \[ \KL{\mG(2\ap)}{\mG(1,2\ap)} = \ap^2 \cdot \log\frac{1}{1-2\ap}. \]
    Using the assumption $\ap\leq1/4$ and the fact $\log\frac{1}{1-x}\leq2x$ for $0\leq x \leq1/2$, we have
    \[ \KL{\mG(\lambda)}{\mG(1,\lambda)} \leq 4\ap^2. \]
  \item
    Fix an arbitrary $\Delta\in[d^4]$.
    By the assumption $d\geq4$ and $\ap\leq1/2$, one has $\lambda\Delta\leq2\ap d^{-1} \leq 1/4$ and hence
    \[
      (1-\lambda)^\Delta
      = \sum_{i=0}^{\Delta}{\Delta \choose i }(-\lambda)^i
      \leq 1 - \lambda\Delta + (\lambda\Delta)^2.
    \]
    Then, it is not difficult to show
    \begin{align*}
      D(\lambda,\Delta) &\leq \lambda^2\Delta\bB{2\Delta+\frac{1}{1-\lambda}} \leq 4\lambda^2\Delta^2 \leq 4\ap^2d^{-2}, \he \\
      D'(\lambda,\Delta) &\leq \ap\bB{2\ap + 5\log{d} + 2\lambda\Delta} \leq 7\ap\cdot\log{d}.
    \end{align*}
    Therefore, we have
    \[ \KL{\mG(\lambda)}{\mG(\Delta,\lambda)} 
    = D(\lambda,\Delta)+D'(\lambda,\Delta)
    \leq 8\ap\cdot\log{d}. \]
  \end{enumerate}
\end{proof}

\subsection{Proof of \texorpdfstring{\cref{lem:lower-bound-d/alpha}}{Lemma 7}} \label{sec:proof-of-d/alpha}
To prove \cref{lem:lower-bound-d/alpha}, we apply Fano's method as introduced in \cref{sec:fano-method}.
We first construct an $\LB{\frac{d}{\ap}}$-spread design matrix $X\in\mat{n}{d}$ and a set $\mB\subset\R^d$ of $\LB{d^d}$ parameter vectors.
We set $m=d/(50\ap)$ throughout \cref{sec:proof-of-d/alpha}.

\paragraph{Design matrix}
Let $\bm R$ be an $m\times d$ Rademacher matrix.
By \cref{thm:sub-Gaussian-matrix-WS}, there exists an absolute constant $c\in(0,1)$ such that $\bm R$ is $\LB{\frac{d}{\ap}}$-spread with high probability for $\ap\leq c$.
Suppose $\ap\leq c$.
Thus, ``most" $m\times d$ $\Set{\pm1}$-matrices are $\LB{\frac{d}{\ap}}$-spread.
Let $Y$ be such a matrix, i.e. $Y\in\Set{\pm1}^{m\times d}$ and $Y$ is $\LB{\frac{d}{\ap}}$-spread.

Let $X_1$ be an arbitrary orthonormal basis matrix of subspace $\col{Y}$.
Then scale $X_1\in\mat{m}{d}$ properly such that $X_1^\T X_1=n\cdot\Id$.
Let $X^\T=\begin{bmatrix} X_1^\T & X_2^\T \end{bmatrix}$ where $X_2$ is a zero matrix.
Then the design matrix $X$ is $\LB{\frac{d}{\ap}}$-spread and satisfies $X^\T X=n\cdot\Id$.

\paragraph{Hard-to-distinguish parameter vectors}
The set of parameter vectors is constructed by reverse engineering.
Let $\ell_{X_1}: \R^d \to \R^{m}$ be a linear mapping defined by $\ell_{X_1}(v):=X_1v$.
We first construct a set $\mU\subset\col{Y}$ with several desired properties and then let $\mB$ be a scaled preimage of $\mU$ under the injective linear mapping $\ell_{X_1}$.
Let $\mU = \Set{Yv ~|~ v\in[d]^d}$. Note that for any $u\in \mU$, we have $u\in\Z^{m}$ and $\Norm{u}_{\infty} \leq d^2$.
Choose the set of parameter vectors to be
\begin{equation} \label{eq:def-B-d/ap}
  \mB = \sigma\cdot\ell_{X_1}^{-1}(\mU) = \sigma\cdot\Set{X_1^{-1}u ~|~ u\in \mU},
\end{equation}
where $\sigma>0$ is a scaling factor. Clearly, $\sigma$ controls the separateness of set $\mB$.
Then for any two distinct vectors $\bt,\bt'\in \mB$, we have
\begin{equation} \label{eq:property-B-d/ap}
  X(\bt-\bt') \in \sigma\cdot\Set{-2d^2,-2d^2+1,...,2d^2}^{m} \times \{0\}^{n-m}.
\end{equation}

We remark that, although the design matrix $X$ we constructed above is rather sparse, it is not necessarily this case and we can easily make $X$ non-sparse via the following trick.
Let $R\in\mat{d}{d}$ be a dense orthogonal matrix, e.g. a uniformly random one.
Now Let $X'=XR$ be the design matrix and $\mB'=\Set{R^\top\bt:\bt\in \mB}$ be the set of parameter vectors.
Clearly, the spreadness of $X'$ is identical to the spreadness of $X$, since $\col{X'}=\col{X}$. Also, $(X')^\top(X') = n\cdot\Id$ and 
\cref{eq:property-B-d/ap} is preserved as well.

\paragraph{Putting things together}
Now we are ready to prove \cref{lem:lower-bound-d/alpha}.
\begin{proof}
  Consider the following hypothesis testing problem.
  Let $D_\beta$ be the uniform distribution over set $\mB$ in \cref{eq:def-B-d/ap} and $\bm\beta^* \sim D_\beta$. 
  Let $X$ be the $\LB{\frac{d}{\ap}}$-spread design matrix as constructed above.
  Set $\lambda=2\ap d^{-5}$ and use the same $\sigma$ in \cref{eq:def-B-d/ap}.
  Let the noise vector be $\bm{\eta} = \bB{\bm\eta_i}_{i=1}^{n}$ where $\bm\eta_1, ..., \bm\eta_n \sim \sigma\cdot\mG(\lambda)$ are independent symmetric geometric random variables as defined in \cref{def:sym_geo_distn}.
  Observing $\bm y = X\bm\bt^* + \bm\eta$, the goal is to distinguish $d^d$ hypotheses $\Set{\bm y=X\bt+\bm\eta : \bt\in\mB}$.
  Now we apply Fano's method by reducing this hypothesis testing problem to oblivious linear regression.

  Given two distinct vectors $\bt,\bt'\in\mB$, let $\Delta_i := \sigma^{-1}\abs{(X\bt)_i-(X\bt')_i}$ for $i\in[n]$. 
  By \cref{eq:property-B-d/ap}, we have $\Delta_i\in\Set{-2d^2,-2d^2+1,...,2d^2}$.
  By independence of random variables $\Set{\bm\eta_i}_{i=1}^{n}$ and the chain rule of Kullback-Leibler divergence, we have
  \begin{align*}
    \KL{X\bt+\bm\eta}{X\bt'+\bm\eta} 
    &= \sum_{i=1}^{n} \KL{(X\bt)_i+\bm\eta_i}{(X\bt')_i+\bm\eta_i} \\
    &= \sum_{i=1}^{m} \KL{\sigma\cdot\mG(\lambda)}{\sigma\cdot\mG(\Delta_i,\lambda)} \\
    &= \sum_{i=1}^{m} \KL{\mG(\lambda)}{\mG(\Delta_i,\lambda)} \\
    &\leq m \cdot 8\ap\log{d}
      = 0.16d\log{d}, \numberthis\label{eq:KL-div-d/ap}
  \end{align*}
  where the second equality uses \cref{eq:property-B-d/ap}, the third equality and the inequality is due to \cref{lemma:facts-sym-geo-distn}.  
  
  Let $\hat{\beta}\,:\R^n\rightarrow\R^d$ be an arbitrary estimator for oblivious linear regression and $\gamma>0$ be an arbitrary given error bound.
  Note $\mB$ is $\sigma\sqrt{{m}/{n}}$-separated and $|\mB|=d^d$.
  Combining \cref{eq:Fano-minimax-lower-bound-real} with \cref{eq:KL-div-d/ap}, and setting $\sigma^2=8\gamma n/m=400\gamma n\ap/d$, we have 
  \begin{equation} \label{eq:d/ap}
    \E \Norm{\hat{\bt}(\bm y)-\bm\bt^*}_2^2 \geq \gamma,
  \end{equation}
  for any $d\geq3$.
\end{proof}

\paragraph{Some remarks}
To show any estimator is inconsistent, it is enough to set $\sigma^2=n\ap/d$ in the above proof.
In this case\footnote{Note that $\sigma^2$ is proportional to the variance of the noise distribution. Setting $\sigma^2=n\ap/d$, then the signal-to-noise ratio does not grow with $n$, which provides one evidence why consistent estimation is impossible.}, the set $\mB$ is $\Omega(1)$-separated and the error lower bound is $\Omega(1)$, which does not vanish as $n$ goes to infinity.
Moreover, since the lower bound \cref{eq:d/ap} holds for any $\gamma>0$, we have actually showed that no estimator can obtain bounded estimation error.

\cref{eq:property-B-d/ap} is crucial in the above proof.
In fact, to prove \cref{lem:lower-bound-d/alpha}, it is enough to construct an $\LB{\frac{d}{\ap}}$-spread design matrix $X\in\mat{n}{d}$ and a set $\mB\subset\R^d$ of parameter vectors such that, for any $\bt,\bt'\in\mB$, one has (i) $X\bt\in\Z^n$, (ii) $\Norm{X(\bt-\bt')}_\infty \leq \poly(d)$, and (iii) $\Norm{X(\bt-\bt')}_0 \lesssim {\log|\mB|}$.

\subsection{Proof of \texorpdfstring{\cref{lem:lower-bound-log{d}/alpha^2}}{Lemma 8}} \label{sec:proof-of-log{d}/alpha^2}
To prove \cref{lem:lower-bound-log{d}/alpha^2}, we apply Fano's method as introduced in \cref{sec:fano-method}.
We first construct an $\LB{\frac{\log{d}}{\ap^2}}$-spread design matrix $X\in\mat{n}{d}$ and a set $\mB\subset\R^d$ of $\LB{d}$ parameter vectors.
We set $k = \log(d)/(200\ap^2)$ throughout \cref{sec:proof-of-log{d}/alpha^2}.

\paragraph{Design matrix}
Pick a random orthogonal matrix $Q\in\R^{d\times d}$.
Let $Y^\T = \begin{bmatrix} Q^\T & Q^\T \end{bmatrix}$.
It is straightforward to see $Y$ is 1-spread.
Let $X_1^\T = \begin{bmatrix} Y_1^\T & \cdots & Y_k^\T \end{bmatrix}$ where $Y_i=Y$ for $i\in[k]$. 
Then $X_1$ is $\LB{\frac{\log{d}}{\ap^2}}$-spread.
Then scale $X_1$ properly such that $X_1^\T X_1=n\cdot\Id$. Obviously, scaling a matrix by a nonzero factor does not change its spreadness.
Let $X^\T=\begin{bmatrix} X_1^\T & X_2^\T \end{bmatrix}$ where $X_2$ is a zero matrix. Note this requires $n \geq k\cdot2d = d\log(d)/(100\ap^2)$.
Then the design matrix $X$ is $\LB{\frac{\log{d}}{\ap^2}}$-spread and satisfies $X^\T X=n\cdot\Id$.

\paragraph{Hard-to-distinguish parameter vectors}
Let $\Set{q_1,...q_d}\subset\R^d$ be the columns of $Q$.
Let
\begin{equation} \label{eq:def-B-log{d}/ap^2}
  \mB = \sigma \sqrt{\frac{k}{n}} \cdot \Set{q_1, ..., q_d}
\end{equation}
be the set of parameter vectors to be distinguish where $\sigma>0$ is a scaling factor. 
The $\sqrt{k/n}$ term in \cref{eq:def-B-log{d}/ap^2} is just to make the subsequent notations cleaner.
It is worth noting that for each $\bt \in \mB$, $X\bt$ is the ``least-spread" vector in $\col{X}$.
For any two distinct vectors $\bt,\bt'\in \mB$, we have
\begin{equation} \label{eq:property-B-log{d}/ap^2}
  X(\bt-\bt') \in \Set{0,\sigma}^n, \quad 
  \Norm{X(\bt-\bt')}_0 = 2k.
\end{equation}
In other words, $X\bt$ and $X\bt'$ differ on exactly $2k$ coordinates and all the differences are equal to $\sigma$.

\paragraph{Putting things together}
Now we are ready to prove \cref{lem:lower-bound-log{d}/alpha^2}.
\begin{proof}
  Consider the following hypothesis testing problem.
  Let $D_\beta$ be the uniform distribution over set $\mB$ in \cref{eq:def-B-log{d}/ap^2} and $\bm\beta^* \sim D_\beta$. 
  Let $X$ be the $\LB{\frac{\log{d}}{\ap^2}}$-spread design matrix as constructed above.
  Let the noise vector be $\bm{\eta} = \bB{\bm\eta_i}_{i=1}^{n}$ where $\bm\eta_1, ..., \bm\eta_n \sim \sigma\cdot\mG(2\ap)$ are independent symmetric geometric random variables as defined in \cref{def:sym_geo_distn}.
  Here the scaling factor $\sigma>0$ is the same $\sigma$ in  \cref{eq:def-B-log{d}/ap^2}.
  Observing $\bm y = X\bm\bt^* + \bm\eta$, the goal is to distinguish $d$ hypotheses $\Set{\bm y=X\bt+\bm\eta : \bt\in\mB}$.
  Now we apply Fano's method by reducing this hypothesis testing problem to oblivious linear regression.

  For any two distinct vectors $\bt,\bt'\in \mB$, by independence of random variables $\Set{\bm\eta_i}_{i=1}^{n}$ and the chain rule of Kullback-Leibler divergence, we have
  \begin{align*}
    \KL{X\bt+\bm\eta}{X\bt'+\bm\eta} 
    &= \sum_{i=1}^{n} \KL{(X\bt)_i+\bm\eta_i}{(X\bt')_i+\bm\eta_i} \\
    &= 2k \cdot \KL{\sigma\cdot\mG(2\ap)}{\sigma\cdot\mG(1,2\ap)} \\
    &= 2k \cdot \KL{\mG(2\ap)}{\mG(1,2\ap)} \\
    &\leq 2k \cdot 4\ap^2 
      = 0.04\log{d}, \numberthis \label{eq:KL-div-log{d}/ap^2}
  \end{align*}
  where the second equality uses \cref{eq:property-B-log{d}/ap^2}, the third equality and the inequality is due to \cref{lemma:facts-sym-geo-distn}.
  
  Let $\hat{\beta}\,:\R^n\rightarrow\R^d$ be an arbitrary estimator for oblivious linear regression and $\gamma>0$ be an arbitrary given error bound.
  Note $\mB$ is $\sigma\sqrt{{2k}/{n}}$-separated and $|\mB|=d$.
  Combining \cref{eq:Fano-minimax-lower-bound-real} with \cref{eq:KL-div-log{d}/ap^2}, and setting $\sigma^2=4\gamma n/k=800\gamma n\ap^2/\log d$, we have 
  \begin{equation} \label{eq:log{d}/ap^2}
    \E \Norm{\hat{\bt}(\bm y)-\bm\bt^*}_2^2 \geq \gamma,
  \end{equation}
  for any $d\geq5$.
\end{proof}

\paragraph{Some remarks}
To show inconsistency, it is enough to set $\sigma^2=n\ap^2/\log d$ in the above proof.
The error lower bound \cref{eq:log{d}/ap^2} can get arbitrarily large.
To prove \cref{lem:lower-bound-log{d}/alpha^2}, it suffices to construct an $\LB{\frac{\log{d}}{\ap^2}}$-spread design matrix $X\in\mat{n}{d}$ and a set $\mB\subset\R^d$ of parameter vectors such that, for any $\bt,\bt'\in\mB$, one has (i) $X\bt\in\Z^n$, (ii) $\Norm{X(\bt-\bt')}_\infty \leq O(1)$, and (iii) $\Norm{X(\bt-\bt')}_0 \lesssim {\log|\mB|}$.

\section{Computational aspects of certifying well-spreadness} \label{section:cc-bounds}

In this section we prove \cref{theorem:main-algorithm} and \cref{theorem:main-lower-bound}.
Concretly, in \cref{section:certifying-well-spreadness}, we provide an efficient algorithm, based on known sum-of-squares algorithms, that can certify an $n\times d$ Gaussian matrices is $\Omega(n)$-spread when $n\gtrsim d^2$.
On the other hand, in \cref{section:avg-case-hardness-well-spreadness}, we provide strong evidence, based on the low-degree polynomial method, which suggests no polynomial-time algorithm is able to certify an $n\times d$ Gaussian matrix is $\Omega(n)$-spread when $n\ll d^2$.

\subsection{Algorithms for certifiying well-spreadness} \label{section:certifying-well-spreadness}

We prove \cref{thm:alg-certify-WS} that shows we can efficiently certify an $n \times d$ Gaussian matrix is $\LB{n}$-spread with high probability whenever $n \gtrsim d^2$.
We consider the regime where $d$ is growing.

\begin{theorem} \label{thm:alg-certify-WS}
  Let $\delta\in(0,1)$ and $C>0$ be arbitrary constants.
  Let $\bm A\sim\Gau{0}{1}^{n\times d}$ with $n \geq Cd^2$.
  There exists a polynomial-time algorithm based on sum-of-squares relaxation and a constant $C'=C'(C)$ such that 
  \begin{enumerate}
    \item if $\bm A$ is not $\bB{(\delta/C')^4n,\delta}$-spread, the algotithm outputs NO;
    \item \whp, $\bm A$ is $\bB{(\delta/C')^4n,\delta}$-spread and the algotithm outputs YES.
  \end{enumerate}
\end{theorem}

\cref{theorem:main-algorithm} is a direct application of \cref{thm:alg-certify-WS} to oblivious linear regression with Gaussian design.

To prove \cref{thm:alg-certify-WS}, we make use of the following result which shows that, \whp, the $2$-to-$4$ norm\footnote{The $p$-to-$q$ norm of a matrix $X$ is defined by $\Norm{X}_{p\to q} := \max_{u\neq0} \Norm{Xu}_q / \Norm{u}_p$.} of an $n \times d$ Gaussian matrix can be efficiently upper bounded by $\UB{n^{1/4}}$, given $n\gtrsim d^2$.

\begin{theorem}[\cite{barak2012hypercontractivity}, Theorem 7.1] \label{thm:alg-certify-2-to-4-norm}
  Let $\bm A\sim\Gau{0}{1}^{n\times d}$.
  There exists a polynomial-time algorithm based on sum-of-squares relaxation that outputs an upper bound $\bm{\mathfrak{U}}$ of the $2$-to-$4$ norm of $\bm A$, i.e. $\max_{\Norm{u}_2=1} \Norm{\bm Au}_4$, which satisfies
  \[ 
    \bm{\mathfrak{U}} \leq n^{1/4} \bB{3 + c\cdot\max\bB{\frac{d}{\sqrt{n}}, \frac{d^2}{n}}}^{1/4} 
  \]
  \whp.
  Here, $c>0$ is an absolute constant.
\end{theorem}

Then it is straightforward to show that, \whp the $\ell_2$-vs-$\ell_4$ distortion of an $n\times d$ Gaussian matrix can be efficiently upper bounded by $O(1)$ given $n\gtrsim d^2$, which we formalize in the following corollary.

\begin{corollary} \label[corollary]{coro:alg-certify-distortion}
  Let $C>0$ be an arbitrary constant.
  Let $\bm A\sim\Gau{0}{1}^{n\times d}$ with $n\geq Cd^2$.
  There exists a polynomial-time algorithm based on sum-of-squares relaxation that outputs an upper bound $\bm{\mathfrak{U}}'$ of the $\ell_2$-vs-$\ell_4$ distortion of $\bm A$, i.e. $\max\Set{n^{1/4}\cdot\Norm{v}_4/\Norm{v}_2 : v\in\col{A}, v\neq0}$,
  which satisfies $\bm{\mathfrak{U}}' \leq C'$ \whp.
  Here $C'>1$ is a constant only depending on $C$.
\end{corollary}
\begin{proof}
  For any non-singular matrix $X\in\mat{n}{d}$, one has
  \begin{align*}
    \Delta_{2,4}(X) 
    &= n^{\frac{1}{4}} \max_{u\neq0} \frac{\Norm{Xu}_4}{\Norm{Xu}_2}
    = n^{\frac{1}{4}} \max_{u\neq0} \frac{\Norm{Xu}_4/\Norm{u}_2}{\Norm{Xu}_2/\Norm{u}_2} \\
    &\leq n^{\frac{1}{4}} \frac{\max_{u\neq0} \Norm{Xu}_4/\Norm{u}_2}{\min_{u\neq0} \Norm{Xu}_2/\Norm{u}_2} \\
    &= \frac{n^{\frac{1}{4}}}{\sigma_{\min}(X)} \cdot \max_{\Norm{u}_2=1} \Norm{Xu}_4.
  \end{align*}
  Now consider $\bm A\sim\Gau{0}{1}^{n \times d}$ which is non-singular almost surely as long as $n\geq d$.
  By \cref{thm:Gaussian-matrix-singular-values-concentrate}, for any $n \gg d$, one has $\sigma_{\min}(\bm A) = \bB{1-o(1)}\sqrt{n}$ \whp. And singular values can be efficiently computed.
  By \cref{thm:alg-certify-2-to-4-norm}, there is an efficiently-computable upper bound $\bm{\mathfrak{U}}$ of $\max_{\Norm{u}_2=1} \Norm{\bm Au}_4$ that satisfies $\bm{\mathfrak{U}} \leq C'' n^{1/4}$ \whp. Here $C''$ only depends on $C$.
  
  Therefore, there exist a constant $C'$ only depending on $C$ and an efficiently-computable upper bound $\bm{\mathfrak{U}}'$ of $\Delta_{2,4}(\bm A)$ such that $\bm{\mathfrak{U}}' \leq C'$ \whp.
\end{proof}

Now, we combine \cref{coro:alg-certify-distortion} and 
\cref{prop:low-distortion-implies-well-spreadness} to prove \cref{thm:alg-certify-WS}.

\begin{proof}
  We first describe the algorithm $\cA$.
  Given an input $\bm{A}\sim\Gau{0}{1}^{n\times d}$, we use the efficient algorithm given by \cref{coro:alg-certify-distortion} to compute an upper bound $\bm{\mathfrak{U}}'$ of the $\ell_2$-vs-$\ell_4$ distortion $\Delta_{2,4}(\bm A)$.
  Let $C'=C'(C)$ be the constant given by \cref{coro:alg-certify-distortion}.
  If $\bm{\mathfrak{U}}'\leq C'$, algorithm $\cA$ outputs YES. Otherwise, algorithm $\cA$ outputs NO.
  
  Then we show algorithm $\cA$ satisfies the two requirements.
  Instantiate \cref{prop:low-distortion-implies-well-spreadness} with $p=2,q=4$ and let $\bm\Delta = \Delta_{2,4}(\bm A)$.
  Then $\bm A$ is $\bB{\delta^4\bm \Delta^{-4}n, \delta}$-spread for any $\delta\in(0,1)$.
  By contrapositivity, if $\bm A$ is not $\bB{(\delta/C')^4 n, \delta}$-spread, then $\bm{\mathfrak{U}}'>C'$ and algorithm $\cA$ will output NO.
  By \cref{coro:alg-certify-distortion}, $\bm{\mathfrak{U}}'\leq C'$ \whp.
  Thus, \whp, $\bm A$ is $\bB{(\delta/C')^4 n, \delta}$-spread and algorithm $\cA$ outputs YES.
\end{proof}

\subsection{Hardness of certifiying well-spreadness} \label{section:avg-case-hardness-well-spreadness}

We provide here formal evidence suggesting the computational hardness of certifiying well-spreadness in average case.
We consider the regime where $d$ is growing and $n\gtrsim d$.

To state our hardness result, we first introduce the noisy Bernoulli-Rademacher distribution (over $\R$) and a distinguishing problem.

\begin{definition}[Noisy Bernoulli-Rademacher distribution] \label[definition]{def:noisy-Ber-Rad-distribution}
  A random variable $\bm x$ following noisy Bernoulli-Rademacher distribution with parameter $\rho\in(0,1)$ and $\sigma\in[0,1/\sqrt{1-\rho})$, denoted by $\bm x\sim\nBR{\rho}{\sigma}$, is defined by
  \[ \bm x =
    \begin{cases} 
      \Gau{0}{\sigma^2}, & \text{with probability} 1-\rho, \\
      +\frac{1}{\sqrt{\rho'}}, & \text{with probability} \frac{\rho}{2}, \\
      -\frac{1}{\sqrt{\rho'}}, & \text{with probability} \frac{\rho}{2},
    \end{cases} \]
  where $\rho' = \frac{\rho}{1-(1-\rho)\sigma^2}$.
\end{definition}

We remark that the particular choice of $\rho'$ in the above definition is to make $\E\bm x^2=1$ for $\bm x\sim\nBR{\rho}{\sigma}$.

\begin{problem}[Distinguishing] \label[problem]{problem:hypothesis-test}
  Let $n,d\in\N$, $\rho\in(0,1)$, and $\sigma\in[0,1/\sqrt{1-\rho})$.
  \begin{itemize}
    \item Under the null distribution $\nu$, observe $\bm A\sim\Gau{0}{1}^{n\times d}$.
    \item Under the planted distribution $\mu$, first sample a hidden vector $\bm v$ whose entries are \iid noisy Bernoulli-Rademacher random variables with parameter $(\rho,\sigma)$.
    Let $\bm Y$ be an $n \times d$ matrix of which the first column is $\bm v$ and the rest entries are independent $\Gau{0}{1}$. Then sample a random orthogonal matrix $\bm Q$ and observe $\bm A = \bm Y \bm Q$. 
  \end{itemize}
  Given a sample $\bm A$ from either $\nu$ or $\mu$, decide from which distribution $\bm A$ is sampled.
\end{problem}

Now we state our computational hardness result.

\begin{theorem} \label{thm:low-deg-hardness-certify-WS}
  Let $\nu$ and $\mu$ be the null an planted distributions defined in \cref{problem:hypothesis-test} respectively.
  Let $C>1$ be an arbitrary constant.
  There exist absolute constants $c_1,c_2,c_3\in(0,1)$ and $C_4>1$ such that the following holds.
  For any $\rho\gg\frac{1}{n}$, $\sigma^2\leq\frac{1}{2}\bB{\log n}^{-C}$, $d\in\bB{C_4\rho^{-1}\sqrt{n}\bB{\log n}^{2C}, c_3n}$, $m\in(1.5\rho n, c_1n)$, constant $\delta\in(c_2,1)$, and $D\leq\bB{\log n}^{C}$, one has
  \begin{enumerate}
    \item $\bm A\sim\nu$ is $(m,\delta)$-spread \whp;
    \item $\bm A\sim\mu$ is not $(m,\delta)$-spread \whp;
    \item $\E_{\nu}\sB{L^{\leq D}(\bm A)^2} \leq O(1)$ where $L^{\leq D}$ is the degree-$D$ likelihood ratio defined in \cref{def:low-deg-likelihood-ratio}.
  \end{enumerate}
\end{theorem}
\begin{proof}
  By \cref{lem:spreadness-property-of-two-distributions} and \cref{lem:bounded-chi-squared-divergence}.
\end{proof}

\paragraph{Implications of \cref{thm:low-deg-hardness-certify-WS}}
Before proving \cref{thm:low-deg-hardness-certify-WS}, we discuss some of its implications. 
First we set $C=2$, $\rho=1/\log n$, and $m=2n/\log n$ in \cref{thm:low-deg-hardness-certify-WS}. It follows that, in the regime where $\sqrt{n} \ll d \lesssim n$, we have
(i) $\bm A \sim \nu$, i.e. $\bm A \sim \Gau{0}{1}^{n \times d}$, is $\LB{n}$-spread \whp;
(ii) $\bm A \sim \mu$ is $o(n)$-spread \whp; and
(iii) it is very likely that no polynomial-time algorithm can distinguish $\nu$ and $\mu$, based on the discussion of low-degree polynomial method in \cref{sec:low-deg-likelihood-ratio}.

Then we apply \cref{thm:low-deg-hardness-certify-WS} to oblivious linear regression with Gaussian design matrix and thus prove \cref{theorem:main-lower-bound}.
By \cite[Theorem 1.2]{d2021consistentICML}, the sufficent conditions for consistent oblivious regression are (i) $n \gg \frac{d}{\ap^2}$ and (ii) the design matrix is $\LB{\frac{d}{\ap^2}}$-spread.
In the following, we will characterize a regime over $(n,d,\ap)$ where 
(i) $n \gg \frac{d}{\ap^2}$; 
(ii) $\bm A\sim\Gau{0}{1}^{n\times d}$ is $\LB{\frac{d}{\ap^2}}$-spread \whp; and
(iii) there exists strong evidence suggesting certifiying $\LB{\frac{d}{\ap^2}}$-spreadness of $\bm A$ is computationally difficult.
To this end, let $n\gg\frac{d}{\ap^2}$ and fix two arbitrary constants $C>0$ and $\delta\in(0,1)$.
Let $\nu$ and $\mu$ be the null and planted distributions considered in \cref{thm:low-deg-hardness-certify-WS}.
It is not difficult to see from the proof of \cref{thm:sub-Gaussian-matrix-WS} that $\bm A\sim\nu$ is $\bB{C\frac{d}{\ap^2}, \delta}$-spread \whp given $n\gg\frac{d}{\ap^2}$.
From the proof of \cref{lem:spreadness-property-of-two-distributions} we know, if $\rho \gg \frac{1}{n}$, $\sigma=o(1)$, and $\rho n \lesssim \frac{d}{\ap^2}$, then $\bm A\sim \mu$ is not $\bB{C\frac{d}{\ap^2}, \delta}$-spread \whp.
Set $C=2$ in \cref{lem:bounded-chi-squared-divergence} and we have the following:
if $\sigma^2\leq\frac{1}{2}\bB{\log n}^{-2}$, $\rho^{-1} n^{1/2} \bB{\log n}^4 \lesssim d$, then $\E_{\nu}\sB{L^{\leq D}(\bm A)^2} \leq O(1)$ for any $D\leq\bB{\log n}^2$.
Therefore, such a regime over $(n,d,\ap)$ can be characterized by
\begin{align*}
  \Set{(n,d,\ap) : \exists \rho \text{ such that } n\gg\frac{d}{\ap^2}, \rho\gg\frac{1}{n}, \rho n\lesssim\frac{d}{\ap^2}, \rho^{-1} n^{1/2} \polylog(n) \lesssim d},
\end{align*}
or equivalently,
\[ \Set{(n,d,\ap) : n^{3/4} \polylog(n)\ap \lesssim d \ll n\ap^2}. \]

Finally, we remark that the ``noiseless" Bernoulli-Rademacher distribution (i.e. $\sigma=0$) already appeared in the literaure (e.g. \cite{d2020sparse, mao2021optimal}).
In the ``noiseless" setting, \cref{problem:hypothesis-test} can be efficiently solved even when $n$ is only linear in $d$ \cite{zadik2021lattice}.
Although the algorithm proposed in \cite{zadik2021lattice} surpasses the lower bound for low-degree polynomial method, their algorithm relies heavily on the extact and brittle structure of the hidden vector.
If we add a little noise to the hidden vector, like what we did here, then their algorithm is likely to fail.

\paragraph{Proof of \cref{thm:low-deg-hardness-certify-WS}}
The following two lemmas together directly imply \cref{thm:low-deg-hardness-certify-WS}.

\begin{lemma} \label[lemma]{lem:spreadness-property-of-two-distributions}
  Let $\nu$ and $\mu$ be the null and planted distributions defined in \cref{problem:hypothesis-test} respectively.
  There exist absolute constants $c_1,c_2,c_3\in(0,1)$ such that the following holds.
  For any $\rho\gg\frac{1}{n}$, $\sigma=o(1)$, $d\leq c_3n$, $m\in(1.5\rho n, c_1n)$, and constant $\delta\in(c_2,1)$, one has
  \begin{enumerate}
    \item $\bm A\sim\nu$ is $(m,\delta)$-spread \whp;
    \item $\bm A\sim\mu$ is not $(m,\delta)$-spread \whp.
  \end{enumerate}
\end{lemma}
\begin{proof}
  The existence of absolute constants $c_1,c_2,c_3$ is guaranteed by \cref{thm:sub-Gaussian-matrix-WS}.
  That is, if $\bm A\sim\Gau{0}{1}^{n \times d}$ and $d \leq c_3n$, then $\bm A$ is $\bB{c_1n,c_2}$-spread with high probability.
  Observe that $(m_1,\delta_1)$-spreadness implies $(m_2,\delta_2)$-spreadness for any $m_2\leq m_1$ and $\delta_2\geq\delta_1$.
  Thus for any $m\leq c_1n$ and $\delta\geq c_2$, $\bm A\sim\nu$ is $(m,\delta)$-spread \whp.
  
  Now consider $\bm A \sim \mu$ and let $\bm v$ be the hidden vector of $\mu$. Clearly, $\bm v \in \col{\bm A}$.
  We decompose $\bm v$ into two parts with disjoint supports, $\bm v = \bm b + \bm \ep$, where $\bm b$ is the Bernoulli-Rademacher part and $\bm \ep$ is the Gaussian part.
  Let $S=\supp{\bm b}$. Then,
  \[ \frac{\Norm{\bm v_S}_2}{\Norm{\bm v}_2} = \frac{\Norm{\bm b}_2}{\Norm{\bm b}_2 + \Norm{\bm \ep}_2}. \]
  By \cref{thm:Chernoff}, 
  \[ \prob{0.5\rho n \leq \Norm{\bm b}_0 \leq 1.5\rho n} \geq 1-2\exp\bB{-\frac{\rho n}{12}}. \]
  By \cref{thm:Gaussian-vector-norm-concentrate}, 
  \[ \prob{\Norm{\bm\ep}_2 \geq 2\sigma\sqrt{n}} \leq \exp\bB{-\frac{n}{2}}. \]
  If $\rho\gg1/n$ and $\sigma=o(1)$, then \whp, we have $|S| \leq 1.5\rho n$ and
  \[ \frac{\Norm{\bm v_S}_2}{\Norm{\bm v}_2} \geq \frac{1}{1+4\sigma} = 1-o(1). \]
  Thus for any $m\geq1.5\rho n $ and any constant $\delta<1$, $\bm A$ is not $(m,\delta)$-spread \whp.
\end{proof}

\begin{lemma} \label[lemma]{lem:bounded-chi-squared-divergence}
  Let $\nu$ and $\mu$ be the null and planted distributions defined in \cref{problem:hypothesis-test} respectively.
  Let $C>1$ be an arbitrary constant.
  For any $D\leq\bB{\log n}^{C}$, $\sigma^2 \leq \frac{1}{2}\bB{\log n}^{-C}$, and $d \geq C_4\rho^{-1}\sqrt{n}\bB{\log n}^{2C}$, one has
  \[ \E_{\nu}\sB{L^{\leq D}(\bm A)^2} \leq O(1), \]
  where $C_4>1$ is an absolute constant and $L^{\leq D}$ is the degree-$D$ likelihood ratio defined in \cref{def:low-deg-likelihood-ratio}.
\end{lemma}

The proof of \cref{lem:bounded-chi-squared-divergence} is an adaptation of the proof of \cite[Theorem 3.4]{mao2021optimal}\footnote{A related result was previously shown in \cite[Theorem 6.7]{d2020sparse}.} which we include here for completeness.
The proof relies on the following three lemmas.

\begin{lemma}[\cite{mao2021optimal}, Lemma 4.23] \label[lemma]{lem:MW21-lemma-4.23}
  Let $\nu$ and $\mu$ be the null and planted distributions defined in \cref{problem:hypothesis-test} respectively.
  Let $\bm u, \bm u'$ be independent uniformly random vectors on the unit sphere in $\R^d$ and $\bm x\sim\nBR{\rho}{\sigma}$.
  Then,
  \[ \E_{\nu} \sB{L^{\leq D}(\bm A)^2} 
  = \sum_{k=0}^{D} \E\inner{\bm u}{\bm u'}^k \sum_{\substack{\ap\in\N^n \\ |\ap|=k}} \prod_{i=1}^{n} \bB{\E h_{\ap_i}(\bm x)}^2, \]
  where $h_k:\R\to\R$ is the $k$-th normalized Hermite polynomial and where $L^{\leq D}$ is the degree-$D$ likelihood ratio defined in \cref{def:low-deg-likelihood-ratio}.
\end{lemma}

\begin{lemma}[\cite{mao2021optimal}, Lemma 4.25] \label[lemma]{lem:MW21-lemma-4.25}
  Let $\bm u$ and $\bm u'$ be independent uniformly random vectors on the unit sphere in $\R^d$.
  For odd $k\in\N$, $\E\inner{\bm u}{\bm u'}^k = 0$.
  For even $k\in\N$, 
  \[ \E\inner{\bm u}{\bm u'}^k \leq (k/d)^{k/2}. \]
\end{lemma}

\begin{lemma}[Adapted from \cite{mao2021optimal}, Lemma 4.26] \label[lemma]{lem:MW21-lemma-4.26}
  For a noisy Bernoulli-Rademacher random variable $\bm x \sim \nBR{\rho}{\sigma}$, we have 
  \begin{enumerate}
    \item $\E h_k(\bm x)=0$ for odd $k\in\N$;
    \item $\E h_0(\bm x)=1$;
    \item $\E h_2(\bm x)=0$;
    \item $\bB{\E h_k(\bm x)}^2 \leq 8^k\rho^{2-k}$ for $k\geq4$ and $\sigma^2 \leq \frac{1}{k-1}$.
  \end{enumerate}
\end{lemma}
\begin{proof}
  Since the noisy Bernoulli-Rademacher distribution is symmetric and odd-degree Hermite polynomials are odd functions, one has $\E h_k(\bm x)=0$ for odd $k\in\N$. 
  It is straightforward to check by definition that 
  \[ \E h_0(\bm x)=1, \quad \E h_2(\bm x) = \frac{1}{\sqrt{2}}\expect{\bm x^2-1} = 0. \]
  Fix an even integer $k\geq4$ and let $\sigma^2 \leq \frac{1}{k-1}$. Then for any even integer $r\in[k]$,
  \[ \E\bm{x}^r 
  = (1-\rho)\sigma^r(r-1)!! + \rho(\rho')^{-r/2} 
  \leq \sigma^r(r-1)!! + \rho^{1-r/2} 
  \leq \sigma^2 + \rho^{1-k/2}
  \leq 2\rho^{1-k/2}. \]
  Also, $\E\bm{x}^0 = 1 \leq 2\rho^{1-k/2}$.
  Let $c_r$ be the coefficient of $z^r$ in the polynomial $\sqrt{k!}\cdot h_k(z)$. Then,
  \[ \abs{\E h_k(\bm x)} 
  = \frac{1}{\sqrt{k!}} \Abs{\sum_{r=0}^{k} c_r \E\bm{x}^r}
  \leq \frac{2\rho^{1-k/2}}{\sqrt{k!}} \sum_{r=0}^{k} \Abs{c_r}. \]
  Define $T(k) := \sum_{r=0}^{k} \Abs{c_r}$. 
  Note that $T(k)$ is the $k$-th telephone number which satisfies the following recurrence,
  \[ T(n) = T(n-1) + (n-1)\cdot T(n-2) \quad \forall n\geq2, \] 
  and $T(0)=T(1)=1$.
  It is easy to show by induction that,
  \[ T(n) \leq C^n n^{n/2}, \quad \forall n\geq1, \quad \forall C\geq\frac{1+\sqrt{5}}{2}. \]
  Now fix some $C\geq\frac{1+\sqrt{5}}{2}$. Using Stirling's approximation, $n!\geq\sqrt{2\pi n}(n/e)^n$ for any $n\geq1$, we have
  \begin{align*}
    \bB{\E h_k(\bm x)}^2
    \leq 4 \rho^{2-k} \cdot \frac{T(k)^2}{k!}
    \leq \frac{4}{\sqrt{2\pi k}} \bB{C^2e}^k \rho^{2-k}.
  \end{align*}
  Therefore, for $k\geq4$ and $\sigma^2 \leq \frac{1}{k-1}$, we have
  \[ \bB{\E h_k(\bm x)}^2 \leq 8^k \rho^{2-k}. \]
\end{proof}

Now we are ready to prove \cref{lem:bounded-chi-squared-divergence}.

\begin{proof}
  Let $\bm x \sim \nBR{\rho}{\sigma}$ be a noisy Bernoulli-Rademacher random variable.
  Given $\ap\in\N^n$, if there exists $i\in[n]$ such that $\ap_i$ is odd or $\ap_i=2$, then $\E h_{\ap_i}(\bm x)=0$ by \cref{lem:MW21-lemma-4.26}.
  Thus, we define the following set
  \[ S(k,m) := \Set{\ap\in\N^n : |\ap|=k, \Norm{\ap}_0=m, \ap_i\in\Set{0}\cup\Set{4,6,8,...} \text{ for all } i\in[n]}. \]
  As $\sigma^2 \leq \frac{1}{2}\bB{\log n}^{-C}$ and $D \leq \bB{\log n}^{C}$, we have $\sigma^2\leq1/(k-1)$ for any $k\leq D$.
  Using \cref{lem:MW21-lemma-4.26}, for $\ap\in S(k,m)$, we have
  \[ \prod_{i=1}^{n} \bB{\E h_{\ap_i}(\bm x)}^2 
  \leq \prod_{\ap_i\neq0} 8^{\ap_i} \rho^{2-\ap_i}
  \leq 8^{k} \rho^{2m-k}. \]
  Note that $S(k,m)$ is empty if $m>\floor{k/4}$. 
  And it is easy to see 
  \[ \Abs{S(k,m)} 
  \leq {n \choose m} m^{k/2}
  \leq n^m \bB{k/4}^{k/2}. \]
  Then, for $k\geq4$, we have
  \begin{align*}
    \sum_{\substack{\ap\in\N^n \\ |\ap|=k}} \prod_{i=1}^{n} \bB{\E h_{\ap_i}(\bm x)}^2
    &= \sum_{m=1}^{\floor{k/4}} \sum_{\ap\in S(k,m)} \prod_{i=1}^{n} \bB{\E h_{\ap_i}(\bm x)}^2 
      \leq \sum_{m=1}^{\floor{k/4}} n^m \bB{k/4}^{k/2} 8^k \rho^{2m-k} \\
    &= \bB{k/4}^{k/2} 8^k \rho^{-k} \frac{\bB{n\rho^2}^{\floor{k/4}+1}-n\rho^2}{n\rho^2-1} 
      \leq \bB{k/4}^{k/2} 8^k \rho^{-k} \frac{\bB{n\rho^2}^{k/4+1}}{n\rho^2/2} \\
    &\leq 2\cdot k^{k/2} n^{k/4} \rho^{-k/2} 4^k.
  \end{align*}
  Let $\bm u$ and $\bm u'$ be independent uniformly random vectors on the unit sphere in $\R^d$.
  Using \cref{lem:MW21-lemma-4.25}, for $k\geq4$, we have
  \[ \E\inner{\bm u}{\bm u'}^k \sum_{\substack{\ap\in\N^n \\ |\ap|=k}} \prod_{i=1}^{n} \bB{\E h_{\ap_i}(\bm x)}^2 
  \leq (k/d)^{k/2} \cdot 2\cdot k^{k/2} n^{k/4} \rho^{-k/2} 4^k 
  = \bB{\frac{512k^4 n}{d^2 \rho^2}}^{k/4}. \]
  Finally, by \cref{lem:MW21-lemma-4.23}, we have
  \begin{align*}
    \E_{\nu} \sB{L^{\leq D}(\bm A)^2} 
    &= \sum_{k=0}^{D} \E\inner{\bm u}{\bm u'}^k \sum_{\substack{\ap\in\N^n \\ |\ap|=k}} \prod_{i=1}^{n} \bB{\E h_{\ap_i}(\bm x)}^2 
      = 1 + \sum_{k\geq4}^{D} \E\inner{\bm u}{\bm u'}^k \sum_{\substack{\ap\in\N^n \\ |\ap|=k}} \prod_{i=1}^{n} \bB{\E h_{\ap_i}(\bm x)}^2 \\
    &\leq 1 + \sum_{k\geq4}^{D} \bB{\frac{512k^4 n}{d^2 \rho^2}}^{k/4}
      \leq 1 + \sum_{k\geq4}^{\infty} \bB{\frac{512n\bB{\log n}^{4C}}{d^2 \rho^2}}^{k/4}.
  \end{align*}
  If there exists a constant $c\in(0,1)$ such that
  $\frac{512n\bB{\log n}^{4C}}{d^2 \rho^2} \leq c$, i.e. $d\geq\sqrt{512/c}\sqrt{n}\bB{\log n}^{2C}$,
  then we have 
  \[ \E_{\nu} \sB{L^{\leq D}(\bm A)^2} \leq O(1). \]
\end{proof}

\section*{Acknowledgments}
We thank David Steurer and Zhihan Jin for helpful discussions.

\phantomsection
\addcontentsline{toc}{section}{Bibliography}
\bibliographystyle{amsalpha}
\bibliography{bib/custom}

\clearpage
\appendix
\section{Concentration bounds}

\begin{theorem}[Chernoff bound] \label{thm:Chernoff}
  Let $\bm X_1, ..., \bm X_n$ be independent Bernoulli ranodm variables with parameter $p$.
  Then for any $t\in[0,np]$,
  \[ \prob{\Abs{\sum_{i=1}^{n}\bm X_i-np} \geq t} \leq 2\exp\bB{-\frac{t^2}{3np}}. \]
\end{theorem}

\begin{theorem} \label{thm:Gaussian-vector-norm-concentrate}
  Let $\bm v\sim\Gau{0}{\Id_n}$. Then for any $t\geq0$, one has
  \[ \prob{\norm{\bm v}_2 \geq \sqrt{n}+t} \leq \exp(-t^2/2). \]  
\end{theorem}

\begin{theorem} \label{thm:Gaussian-matrix-singular-values-concentrate}
  Let $\bm A\sim\Gau{0}{1}^{n\times d}$.
  Then for any $t\geq0$, with probability at least $1-2\exp(-t^2/2)$,
  \[ \sqrt{n}-\sqrt{d}-t \leq \sigma_{\min}(\bm A) \leq \sigma_{\max}(\bm A) \leq \sqrt{n}+\sqrt{d}+t. \]
\end{theorem}

\begin{definition}[Sub-Gaussian norm] \label[definition]{def:sub-Gaussian-norm}
  The sub-Gaussian norm of a $d$-dimensional random vector $\bm x$ is defined by
  \[ \Orlicz{2}{\bm x} := \sup_{\substack{v\in\R^d \\ \norm{v}_2=1}} \inf\Set{t>0 : \E\exp\bB{\frac{\inner{\bm x}{v}^2}{t^2}} \leq 2}. \]
\end{definition}

\begin{theorem}[\cite{vershynin2018high}, Theorem 4.6.1] \label{thm:sub-Gaussian-matrix-singular-values-concentrate}
  Let $\bm A$ be an $n \times d$ random matrix with independent rows $\bm A_1, ..., \bm A_n$. 
  Suppose $\bm A_i$'s have zero mean, identity covariance matrix, and $K := \max_{i\in[n]} \Orlicz{2}{\bm A_i} < \infty$ (see \cref{def:sub-Gaussian-norm}). 
  Then for any $t\geq0$, with probability at least $1-2\exp(-t^2)$,
  \[ \sqrt{n}-CK^2\bB{\sqrt{d}+t} \leq \sigma_{\min}(\bm A) \leq \sigma_{\max}(\bm A) \leq \sqrt{n}+CK^2\bB{\sqrt{d}+t}, \]
  where $C>0$ is an absolute constant.
\end{theorem}

\begin{theorem}[Well-spreadness of sub-Gaussian matrices] \label{thm:sub-Gaussian-matrix-WS}
  Let $\bm A$ be an $n \times d$ random matrix with independent rows $\bm A_1, ..., \bm A_n$. 
  Suppose $\bm A_i$'s have zero mean, identity covariance, and $K:=\max_{i\in[n]} \Orlicz{2}{\bm A_i} \leq O(1)$ (see \cref{def:sub-Gaussian-norm}). 
  Then there exist absolute constants $c_1,c_1,c_3\in(0,1)$ such that $\bm A$ is $\bB{c_1n, c_2}$-spread with probability at least $1-\exp\bB{-\LB{n}}$ for $d \leq c_3n$.
\end{theorem}

\begin{proof}
  In the following, $c_1,c_2,c_3,c_4,c_5\in(0,1)$ are sufficiently small constants that only depend on $K$ and the absolute constant $C$ in \cref{thm:sub-Gaussian-matrix-singular-values-concentrate}.
  Suppose $d\leq c_3n$ and let $k=c_1n$. 
  We will show that, \whp, for any nonzero $v\in\R^{d}$ and any $S\subset[n]$ with $|S|=k$, one has $\Norm{\bm A_Sv}_2 \leq c_2\Norm{\bm Av}_2$, where $\bm A_S$ is a $|S|\times d$ submatrix of $\bm A$ with rows indexed by $S$. 
  
  Fix a set $S\subset[n]$ with $|S|=k$. By \cref{thm:sub-Gaussian-matrix-singular-values-concentrate}, with probability at least $1-2\exp\bB{-c_4n}$,
  \[ 
    \sigma_{\max}(\bm A_S) 
    \leq \sqrt{k} + C'\bB{\sqrt{d}+\sqrt{c_4n}} 
    \leq \bB{\sqrt{c_1} + C'\bB{\sqrt{c_3}+\sqrt{c_4}}} \sqrt{n}, 
  \]
  where $C' = CK^2$ is a constant.
  Using ${n \choose k} \leq \bB{\frac{en}{k}}^k$ and applying union bound, we have with probability at least $1-2\exp\hB{\bB{-c_4+c_1(1-\log{c_1})}n}$ that,
  \[ 
    \sigma_{\max}(\bm A_S)
    \leq \bB{\sqrt{c_1} + C'\bB{\sqrt{c_3}+\sqrt{c_4}}} \sqrt{n}
  \]
  for any $S\subset[n]$ with $|S|=k$.
  Using \cref{thm:sub-Gaussian-matrix-singular-values-concentrate} again, we have
  \[ 
    \sigma_{\min}(\bm A) 
    \geq \sqrt{n}-C'\bB{\sqrt{d}+\sqrt{c_5n}}
    \geq \bB{1 - C'\bB{\sqrt{c_3}+\sqrt{c_5}}} \sqrt{n}
  \]
  with probability at least $1-2\exp\bB{-c_5n}$.
  
  Given any constant $C'>0$, we can always choose sufficiently small constants $c_1,c_2,c_3,c_4,c_5\in(0,1)$ such that 
  (i) $\frac{\sqrt{c_1} + C'\bB{\sqrt{c_3}+\sqrt{c_4}}}{1-C'\bB{\sqrt{c_3}+\sqrt{c_5}}} \leq c_2$ and
  (ii) $-c_4+c_1(1-\log{c_1})<0$.
  Then with probability at least $1-\exp\bB{-\LB{n}}$, one has for any nonzero $v\in\R^{d}$ and any $S\subset[n]$ with $|S|=k$ that,
  \[ \frac{\Norm{\bm A_Sv}_2}{\Norm{\bm Av}_2} 
  \leq \frac{\sigma_{\max}(\bm A_S)}{\sigma_{\min}(\bm A)} \leq c_2. \]
\end{proof}

\begin{remark}
  For a random matrix with \iid standard Gaussian or Rademacher random variables, it is easy to check $K\leq O(1)$.
\end{remark}
\section{NP-hardness of deciding well-spreadness} \label{sec:NP-hardness-deciding-well-spreadness}

We prove \cref{thm:decide-spreadness-is-np-hard} that shows deciding whether a matrix satisfies a given well-spreadness condition is NP-hard.
To cope with computational complexity issues with numbers, we will assume all input numbers to be rational.
For a rational number $r\in\Q$, let $\aB{r}$ denote its encoding length, i.e. the length of its representation.
For a rational matrix $A\in\Q^{n \times d}$, let $\aB{A} := \sum_{i=1}^{n} \sum_{j=1}^{d} \aB{A_{ij}}$ denote its encoding length.

\begin{problem} \label[problem]{problem:decide-well-spreadness}
  Given as input $A\in\Q^{n\times d}$, $m\in[n]$, and $\delta\in\Q$, decide whether $A$ is $(m,\delta)$-spread.
\end{problem}

\begin{theorem} \label{thm:decide-spreadness-is-np-hard}
  \cref{problem:decide-well-spreadness} is NP-hard.
\end{theorem}

To prove \cref{thm:decide-spreadness-is-np-hard}, we will show the following problem is NP-hard and there exists a polynomial-time reduction from this problem to \cref{problem:decide-well-spreadness}.

\begin{problem} \label[problem]{problem:decide-well-spreadness-kernel}
  Given as input $A\in\Q^{p\times n}$, $m\in[n]$, and $\delta\in\Q$, decide whether $\ker{A}$ is $(m,\delta)$-spread.
\end{problem}

Following \cite{bandeira2013certifying, tillmann2014computational}, our proof of the NP-hardness of \cref{problem:decide-well-spreadness-kernel} is based on a reduction from the problem of deciding \emph{matrix spark} (\cref{problem:spark}).

\begin{definition}[Matrix spark] \label[definition]{def:spark}
  The spark of a matrix $A$ is the smallest number $k$ such that there exists a set of $k$ columns of $A$ that are linearly dependent. 
  Equivalently,
  \[ \spark{A} := \min\Set{\Norm{x}_0 : Ax=0, x\neq0}. \]
\end{definition}

\begin{problem} \label[problem]{problem:spark} 
  Given as input $A\in\Q^{p\times n}$ and $m\in\N$, decide whether $\spark{A}>m$.
\end{problem}

By a reduction from the NP-complete $k$-clique problem, i.e. deciding whether a given simple graph has a clique of size $k$,  \cref{problem:spark} is proven to be NP-hard in \cite{tillmann2014computational}.
Moreover, the matrices in the hard instances of \cref{problem:spark} are integer matrices whose entry-wise encoding length is bounded by a polynomial in $p$ and $n$.

\begin{theorem} \label{thm:decide-spreadness-kernel-is-np-hard}
  \cref{problem:decide-well-spreadness-kernel} is NP-hard.
\end{theorem}

\begin{proof}
  Let $(A,m)$ be a hard instance of \cref{problem:spark} given by \cite{tillmann2014computational}.
  Let $P=\Norm{A}_{\infty}$. It is known that $\aB{P}$ is bounded by some polynomial in $p$ and $n$.
  Our strategy is to choose an appropriate rational number $\delta\in(0,1)$ with $\aB{\delta}$ bounded by some polynomial in $p$ and $n$ such that the following is true.
  When we give the instance ($A,m,\delta$) to an oracle of \cref{problem:decide-well-spreadness-kernel}, 
  if the answer is YES, then $\spark{A}>m$; 
  if the answer is NO, then $\spark{A}\leq m$.
  If such a $\delta$ exists, then we have a polynomial-time reduction from \cref{problem:spark} to \cref{problem:decide-well-spreadness-kernel}, and as a result, \cref{problem:decide-well-spreadness-kernel} is NP-hard.

  In the following, we show how to construct such a $\delta\in(0,1)$.
  For the case when the oracle answers YES, it is straightforward to see $\spark{A}>m$ for any $\delta\in(0,1)$.
  For the case when the oracle answers NO, we consider the contrapositive. 
  Assume $\spark{A}>m$.
  We want a $\delta\in(0,1)$ such that $\Norm{v_{S}}_2\leq\delta\Norm{v}_2$ for any nonzero $v\in\ker{A}$ and any $S\subseteq[n]$ with $|S|\leq m$.
  
  Take an arbitrary nonzero vector $x\in\ker{A}$. Without loss of generality, assume $\abs{x_1} \geq \abs{x_2} \geq \cdots \geq \abs{x_n}$. Let $S=[m]$ and $\bar{S}=[n]\sm S$. Then it suffices to upper bound $\Norm{x_{S}}_2/\Norm{x}_2$ by $\delta$.
  Let $A_S$ be the $m\times k$ submatrix of $A$ with columns indexed by $S$ and define $A_{\bar{S}}$ likewise. 
  Since $x\in\ker{A}$, we have
  \begin{align*}
    Ax=0 &\iff A_Sx_S + A_{\bar{S}}x_{\bar{S}} = 0  \\
    &\implies \Norm{A_Sx_S}_2 = \Norm{A_{\bar{S}}x_{\bar{S}}}_2 \\
    &\implies \frac{\Norm{x_{S}}_2}{\Norm{x_{\bar{S}}}_2} \leq \frac{\sigma_{\max}\bB{A_{\bar{S}}}}{\sigma_{\min}\bB{A_S}}.
  \end{align*} 
  It is easy to see
  \[ \sigma_{\max}\bB{A_{\bar{S}}} \leq \Norm{A_{\bar{S}}}_F \leq \Norm{A}_F \leq \sqrt{pn}\cdot P. \]
  From the proof of \cite[Theorem 4]{bandeira2013certifying}, we know 
  \[ \sigma_{\min}\bB{A_{S}}^2 \geq \bB{pmP^2}^{1-m} \geq \bB{pnP^2}^{1-p}. \]
  Therefore,
  \begin{align*}
    \frac{\Norm{x_S}_2}{\Norm{x}_2} 
    \leq \sqrt{\frac{1}{1+\sigma_{\min}\bB{A_S}^2/\sigma_{\max}\bB{A_{\bar{S}}}^2}}
    \leq \sqrt{\frac{1}{1+\bB{pnP^2}^{-p}}}
    \leq 1 - \frac{1}{2\bB{\bB{pnP^2}^p+1}}.
  \end{align*}
  Set $\delta = 1 - \frac{1}{2\bB{\bB{pnP^2}^p+1}}$. 
  Then $\ker{A}$ is $(m,\delta)$-spread.
  Moreover, $\aB{\delta}\leq f(p,n)$ for some polynomial $f$.
\end{proof}

Now we are ready to prove \cref{thm:decide-spreadness-is-np-hard}.

\begin{proof}
  Given \cref{thm:decide-spreadness-kernel-is-np-hard}, it only remains to show there exists a polynomial-time reduction from \cref{problem:decide-well-spreadness-kernel} to \cref{problem:decide-well-spreadness}.
  It is well-known that, given as input a matrix $X\in\Q^{p\times n}$, Gaussian elimination is able to produce in polynomial time a matrix $Y\in\Q^{n\times(n-p)}$ such that $\ker{X}=\col{Y}$ and $\aB{Y}$ is polynomial in $\aB{X}$.
\end{proof}

\cref{thm:decide-spreadness-is-np-hard} establishes the NP-hardness of deciding whether a given matrix is $(m,\delta)$-spreadness when $m$ and $\delta$ are also inputs.
Nevertheless, this result has a major limitation in the conext of oblivious regression.
That is, the parameter $\delta\in(0,1)$ used in the above proof is $1-o(1)$ that this result reveals almost nothing about the hardness of the more interesting case when $\delta$ is a constant.

\end{document}